\documentclass{article}

\usepackage{times}
\usepackage{graphicx} 
\usepackage{caption}
\usepackage{subcaption}

\usepackage{natbib}

\usepackage{algorithm}
\usepackage[noend]{algorithmic}

\usepackage{hyperref}

\usepackage[accepted]{icml2016} 

\usepackage{savesym}
\usepackage{multirow}
\usepackage{xspace}
\usepackage{graphicx}
\usepackage{url}
\usepackage{enumitem}
\usepackage{color}
\usepackage{amsmath,amssymb,xfrac,amsthm,amsfonts,bbm}
\usepackage{verbatim, xfrac}
\usepackage[titletoc]{appendix}

\usepackage{subscript}

\usepackage[all=normal,floats,wordspacing,paragraphs]{savetrees}

\setitemize{noitemsep,topsep=1pt,parsep=1pt,partopsep=1pt, leftmargin=12pt}
\setlist{nolistsep}
\newtheorem{lemma}{Lemma}

\newtheorem*{example*}{Example}
\newtheorem{theorem}{Theorem}

\newtheorem{corollary}{Corollary}

\newtheorem*{theorem*}{Theorem}

\newcommand{\Rmnum}[1]{\expandafter\@slowromancap\romannumeral #1@}
\makeatother

\usepackage{titlesec}
\titlespacing*{\section}{0pt}{-0.20\baselineskip}{-0.350\baselineskip}
\titlespacing*{\subsection}{0pt}{-2.5pt}{-7.0pt}
\expandafter\def\expandafter\normalsize\expandafter{%
    \normalsize
    \setlength\abovedisplayskip{2.0pt}
    \setlength\belowdisplayskip{2.0pt}
    \setlength\abovedisplayshortskip{2pt}
    \setlength\belowdisplayshortskip{2pt}
}

\usepackage{subfloat}
\makeatletter

\DeclareMathOperator*{\argmin}{arg\,min}
\DeclareMathOperator*{\argmax}{arg\,max}

\usepackage{bbm}
\usepackage{mathtools}
\usepackage{epstopdf}
\DeclarePairedDelimiterX{\norm}[1]{\lVert}{\rVert}{#1}
\newcommand{\normsize}{\ensuremath{1}}

\newcommand{\bigObound}{\ensuremath{\mathcal{O}}}
\newcommand{\Thetabound}{\ensuremath{\Theta}}
\newcommand{\R}{\ensuremath{\mathbb{R}}}
\newcommand{\indfunc}{\ensuremath{\mathbf{1}}}

\newcommand{\itemset}{\ensuremath{\mathcal{A}}}
\newcommand{\num}{\ensuremath{n}}
\newcommand{\dataset}{\ensuremath{\mathcal{Z}}}
\newcommand{\datapoint}{\ensuremath{z}}
\newcommand{\queryset}{\ensuremath{\mathcal{X}}}
\newcommand{\querypoint}{\ensuremath{x}}
\newcommand{\labelpoint}{\ensuremath{y}}
\newcommand{\labelfunc}{\ensuremath{Y}}
\newcommand{\hypclass}{\ensuremath{\mathcal{D}}}
\newcommand{\hyp}{\ensuremath{D}}
\newcommand{\hypstar}{\ensuremath{\hyp^*}}
\newcommand{\hyphat}{\ensuremath{\hat{\hyp}}}
\newcommand{\Dhat}{\ensuremath{\hat{D}}}
\newcommand{\D}{\ensuremath{D}}
\newcommand{\uclass}{\ensuremath{\hypclass}}
\newcommand{\lclass}{\ensuremath{\mathcal{L}}}
\newcommand{\closure}{\ensuremath{\textnormal{cl}}}
\newcommand{\clique}{\ensuremath{\mathcal{C}}}

\newcommand{\cost}{\ensuremath{c}}
\newcommand{\clusters}{\ensuremath{K}}
\newcommand{\posr}{\ensuremath{\text{positive}}}
\newcommand{\negr}{\ensuremath{\text{negative}}}
\newcommand{\range}{\ensuremath{r}}
\newcommand{\rangeintra}{\ensuremath{r_{\textnormal{in}}}}
\newcommand{\prob}{\ensuremath{\mathbb{P}}}
\newcommand{\Time}{\ensuremath{m}}

\newcommand{\cuisine}{\ensuremath{\text{cuisine}}}
\newcommand{\review}{\ensuremath{\text{review}}}
\newcommand{\longitude}{\ensuremath{\text{long}}}
\newcommand{\latitude}{\ensuremath{\text{lat}}}
\newcommand{\geo}{\ensuremath{\text{geo}}}
\newcommand{\rand}{\ensuremath{\text{random}}}

\newcommand{\naive}{\textsc{IndGreedy}\xspace}
\newcommand{\naiveSIT}{\textsc{IndGreedy-SIT}\xspace}
\newcommand{\learnhm}{\textsc{LearnHM}\xspace}
\newcommand{\luproj}{\textsc{LU-Proj}\xspace}
\newcommand{\uproj}{\textsc{U-Proj}\xspace}
\newcommand{\lproj}{\textsc{L-Proj}\xspace}

\newcommand{\getuserresponse}{\textsc{GetUserResponse}\xspace}

\newcommand{\qclique}{\textsc{QClique}\xspace}
\newcommand{\qgreedy}{\textsc{QGreedy}\xspace}

\icmltitlerunning{Actively Learning Hemimetrics}

\begin{document}

\twocolumn[
\icmltitle{Actively Learning Hemimetrics \\ with Applications to Eliciting User Preferences}

\icmlkeywords{hemimetrics, metrics, active learning, crowd preferences}

\icmlauthor{Adish Singla}{adish.singla@inf.ethz.ch}
\icmlauthor{Sebastian Tschiatschek}{sebastian.tschiatschek@inf.ethz.ch}
\icmlauthor{Andreas Krause}{krausea@ethz.ch}
\icmladdress{ETH Zurich, Switzerland}

\vskip 0.3in
]
\begin{abstract}
Motivated by an application of eliciting users' preferences, we investigate the problem of learning hemimetrics, {\em i.e.}, pairwise distances among a set of $\num$ items that satisfy triangle inequalities and non-negativity constraints. In our application, the (asymmetric) distances quantify private costs a user incurs when substituting one item by another. We aim to learn these distances (costs) by asking the users whether they are willing to switch from one item to another for a given incentive offer.
Without exploiting structural constraints of the hemimetric polytope, 
learning the distances between each pair of items requires $\Thetabound(\num^2)$ queries. We propose an active learning algorithm that substantially reduces this sample complexity by exploiting the structural constraints on the version space of hemimetrics. Our proposed algorithm achieves provably-optimal sample complexity for various instances of the task. For example, when the items are embedded into $\clusters$ tight clusters, the sample complexity of our algorithm reduces to $\bigObound(\num\clusters)$.  Extensive experiments on a restaurant recommendation data set 
support the conclusions of our theoretical analysis.
\end{abstract}
\vspace{-5mm}
%


\section{Introduction}\label{sec.introduction}
Learning a distance function over a set of items or a data manifold plays a crucial role in many real-world applications. In machine learning algorithms, the distances serve as a notion of similarity (or dissimilarity) between data points and are important for various tasks such as clustering \cite{xing2002distance}, object ranking \cite{lim2014efficient}, image retrieval / classification \cite{he2004manifold,huang2015log}, 
etc.\footnote{We refer the interested reader to the survey by~\citet{bellet2013survey} for a detailed discussion of various applications.}
In economics, the distance function can encode the preferences of users ({\em e.g.}, buyers or sellers in a marketplace) for different items ({\em e.g.}, from a catalogue of products) to improve product recommendation and dynamic pricing of goods \cite{desjardins2006learning,horton2015quality,blum2015learning,vazquez2014active}.
\\
{\bf Motivating applications.}
We are interested in learning the preferences of users across different choices available in a marketplace --- these choices are given in the form of $\num$ (types of) items. For instance, in a restaurant recommendation system such as \emph{Yelp}, the item types could correspond to restaurants abstracted by attributes such as cuisine, locality, reviews and so on. Consider a user who seeks recommendations from the system and has chosen item $i$ ({\em e.g.}, ``Mexican restaurant in Manhattan with over $50$ reviews"). However, to incentivize exploration and maximize its overall utility, the marketplace may consider offering a discount to the user to instead choose item $j$ ({\em e.g.}, ``Newly opened fastfood restaurant in New Jersey with $0$ reviews"), \emph{e.g.}, to gather more reviews for item $j$. The price of the offer would clearly depend on how similar or dissimilar the choices $i$ and $j$ are. In general, a high dissimilarity would require the system to offer higher incentives (larger discounts). 
\\
{\bf Distance function quantifying private costs.}
We capture the above-mentioned notion of dissimilarity by a pairwise distance $\D_{i,j}$ --- the distance $\D_{i,j}$ corresponds to the private cost of a user incurred by switching from her default choice of item $i$ to item $j$.  We assume that the distance function $\D$ is a {\em hemimetric}, {\em i.e.}, a relaxed form of a metric, satisfying only non-negativity constraints 
and triangle inequalities. 
The asymmetry in the distances ({\em i.e.}, $\D_{i,j} \neq \D_{j,i}$) is naturally required in our application setting --- it may arise from various factors such as the underlying quality of the items ({\em e.g.}, switching between a highly rated and an unreviewed restaurant).
Our goal is to efficiently learn the distance function $\D$ via interactions with the users, without assuming any knowledge of the underlying attributes that affect the users' preferences.
\\
{\bf User query and active learning.}
In our setting, the interactions with the users take the form of a binary labeling query, \emph{i.e.}, given two items $i$ and $j$, and a proposed value $\cost$, the user provides a $\posr$ response iff value $\cost$ is at least the underlying distance $\D_{i,j}$, and a $\negr$ response otherwise. This query is motivated by the \emph{posted-price model} used in marketplaces  \cite{abernethy2015low,singla2013truthful},  where users are offered a take-it-or-leave-it price by the system, and they can accept by providing a $\posr$ response, or reject by providing a $\negr$ response.  However, we are not considering the economic aspects of the query, and each query has unit cost for the algorithm. Such a feedback setting is realistic and often employed by online marketplaces where queries are posted to the users in form of surveys to seek feedback or the users who accept are awarded the actual monetary offer via a lottery. This paper is concerned with the sample complexity, \emph{i.e.}, the number of queries required to learn $\D$.
\subsection{Our Approach and Main Contributions}
A naive approach to this problem is to learn each of the $\num^2$ pairwise distances independently. However, the sample complexity of this approach is $\Thetabound(\num^2)$. Our goal is to reduce the sample complexity by exploiting the structural constraints  on the version space of hemimetrics. The main contributions of this paper are as follows: 
\\
{\bf Novel metric learning framework.} We propose a new learning problem in the framework of metric learning, motivated by the applications to eliciting users' preferences. The key distinctive features of our setting 
are: \emph{(i)} the specific modality of the user queries (with natural motivation from economics), and \emph{(ii)} the asymmetry of the distances learnt.
\\
{\bf Exploiting structural constraints.} We develop a novel active learning algorithm $\learnhm$ that substantially reduces the sample complexity by exploiting the structural constraints of the hemimetric polytope. We provide tight theoretical guarantees on the sample complexity of the proposed algorithm for several natural problem instances.
\\
{\bf Practical extensions.} Our algorithm extends to various important practical settings, including: \emph{(i)} the online setting where the $\num$ items are not known beforehand, and rather appear over time, and \emph{(ii)} the noisy setting that reflects the stochastic nature of acceptance of offers by the users.
\section{Problem Statement}\label{sec.model}
We now formalize the problem addressed in this paper.
{\bf Items $\itemset$.}
There are $\num$ items (or types of items), denoted by the set $\itemset=\{1, 2, \ldots, \num\}$.  For instance, in a restaurant recommendation system such as \emph{Yelp}, $\itemset$ could consist of types of restaurants distinguished by high-level attributes such as cuisine, locality, reviews and so on. 
\\
{\bf Hemimetric $\hypstar$.}
Let $\hypclass$ be the set of bounded {\em hemimetrics}, \emph{i.e.}, matrices $\D\in\mathbb{R}^{\num \times \num}$ that satisfy
\begin{align}
\D_{i,i} &= 0 \quad \forall \ i \in [n],\\
\D_{i,j} &\geq 0 \quad \forall \ i, j \in [n],\\
\D_{i,j} &\leq \range \quad \forall \ i, j \in [n], \textrm{ and } \\
\D_{i,j} &\leq \D_{i,k} + \D_{k,j} \quad \forall \ i, j, k \in [n], \label{eq.triangle.inequality}
\end{align}
where $[n] = \{1, 2, \ldots, n\}$ and $\range$ is the upper bound on the value.
We assume that user preferences are represented by an underlying \emph{unknown} hemimetric $\hypstar \in \hypclass$. The (asymmetric) distance $\hypstar_{i,j}$ quantifies the private costs of the user for switching from item $i$ to $j$. 
Our goal is to learn $\hypstar$ via interactions with the users, without assuming any knowledge of the underlying attributes that affect the user preferences. 
\\
{\bf User query.}
A query $\querypoint \in \queryset \coloneqq \{ (i,j,\cost) \mid i,j \in [n], \cost \in [0, r] \}$ to the user is characterized by the tuple $(i,j,\cost)$, where item $i$ denotes the choice of the user, item $j$ denotes the alternative suggested by the algorithm, and $\cost$ denotes the monetary incentives offer. Note that the notion of \emph{user} as used in this paper is rather generic, and could correspond to one single user or a crowd / cohort of users. 
\\
{\bf User response.}
The response to a query (also called label) is denoted as $\labelpoint = \labelfunc(x)$, where $\labelfunc \colon \queryset \mapsto \{0, 1\}$. A label $\labelpoint=1$ indicates acceptance, while $\labelpoint=0$ indicates rejection.
We denote a labeled datapoint by $\datapoint = (\querypoint, \labelpoint)$, where $\querypoint \in \queryset$, and $\labelpoint = \labelfunc(\querypoint)$. In our setting the acceptance function is stochastic. It is characterized by $\prob(\labelfunc(\querypoint)=\labelpoint)$ and is required to satisfy the following two mild conditions related to the decision boundary at $\hypstar_{i,j}$ and monotonicity: 
\vspace{1mm} 
\begin{align*}
  &\quad \prob(\labelfunc((i,j,\cost))=1) \geq 0.5 \quad \textrm{iff } \cost \geq \hypstar_{i, j}, \textrm{ and}\\
  &\quad \prob(\labelfunc((i,j,\cost))=1) \geq \prob(\labelfunc((i,j,\cost'))=1) \quad \textrm{for } \cost \geq \cost'.
\vspace{2mm}
\end{align*}
For ease of exposition of our main results, we focus on a deterministic noise-free setting in the main paper --- treatment of the (more realistic) stochastic acceptance function is presented in Appendix~\ref{app.algoPart4-Extensions-Noise}. 
In this noise-free setting, the acceptance function reduces to the threshold function
\vspace{1mm}
\begin{align}
\labelfunc((i,j,\cost)) = 
\begin{cases}
1 & \text{ if } \cost \geq \hypstar_{i, j}, \\
0 & \text{ otherwise. } 
\end{cases}
\label{eq.querylabel.D}
\end{align}
\vspace{1mm}
\\
{\bf Objective.}
This paper is concerned with the sample complexity, {\em i.e.}, the number of queries required for learning the unknown hemimetric $\hypstar$. 
We consider a PAC-style setting, \emph{i.e.}, we aim to design an algorithm that, given positive constants $(\epsilon, \delta)$, determines a hemimetric $\hyphat \in \hypclass$, such that with probability at least $1 - \delta$ it holds that  
\vspace{1mm}
\begin{align}
  \| \hyphat - \hypstar \|_{\infty} \leq \epsilon, \label{eq.objectiveInfinityNorm}
\end{align}
\emph{i.e.}, $|\hyphat_{i,j} - \hypstar_{i,j}| \leq \epsilon \quad \forall \ i, j \in [\num]$.

Our objective is to achieve the desired $(\epsilon, \delta)$-PAC guarantee while minimizing the number of user queries.

\section{Warmup: Overview of our Approach}\label{sec.algoPart1-Ideas}
We now present the high-level ideas behind our approach. 

\subsection{Independent Learning: \naive} \label{sec.algoPart1-Ideas.naive}
One possible way to tackle our learning problem is to learn each of the $\num^2$ pairwise distances independently. Let us fix a particular pair of items $(i,j) \in [\num]^2$. Given the query modality considered in our framework, the goal of learning the distance $\hypstar_{i,j}$ up to precsion $\epsilon$ is equivalent to learning a threshold function in the active-learning setting \cite{castro2006upper,settles2012active}. In terms of sample complexity, this is most effectively achieved by perfoming a binary-search over the range $[0, \range]$. More formally, at iteration $t=0$, we initialize a lower bound of $\hypstar_{i,j}$ to $L^t_{i,j} = 0$ and an upper bound to $U^t_{i,j} = \range$. At any $t > 0$, we pick a value $\cost^t =  \tfrac{1}{2} ( L^{t-1}_{i,j} + U^{t-1}_{i,j})$ and issue the query $\querypoint^t = (i, j, \cost^t)$. Then, based on the returned label $\labelpoint^t$, we update $U^{t}_{i, j} = \cost^t$ if $\labelpoint^t = 1$, otherwise we update $L^{t}_{i, j} = \cost^t$ if $\labelpoint^t = 0$. We continue querying until $(U^{t}_{i,j} - L^{t}_{i,j}) \leq \epsilon$, and then output any $\Dhat_{i,j} \in [L^t_{i,j}, U^t_{i,j}]$ as the estimated distance. The number of queries required in the noise-free setting is given by $\lceil\log(\frac{\range}{\epsilon})\rceil$. As there are $\num^2$ pairwise distance learning problems, the total sample complexity of this approach is  $\num^2 \lceil\log(\frac{\range}{\epsilon})\rceil$. \\
Such an algorithm, based on independently learning the pairwise distances, also needs to pick a pair $(i^t, j^t)$ to query at iteration $t$. One policy inspired by uncertainty sampling~\cite{settles2012active} is to pick the pair $(i^t, j^t)$ with maximum uncertainty quantified by $(  U^{t-1}_{i,j} - L^{t-1}_{i,j} )$. This policy can also be seen as to greedily minimize our objective stated in Equation~\ref{eq.objectiveInfinityNorm}. We call this query policy \qgreedy. At any iteration, it issues the query $\querypoint^t = (i^t, j^t, \cost^t)$ according to
\vspace{2mm}
\begin{align}
(i^t, j^t) &=  \argmax_{(i, j) \in  [\num]^2} (  U^{t-1}_{i,j} - L^{t-1}_{i,j} ), \textrm{ and} \label{eq.qgreedy1}\\
\cost^t &=  \tfrac{1}{2} (L^{t-1}_{i^t,j^t} + U^{t-1}_{i^t,j^t} ). \label{eq.qgreedy2}
\vspace{6mm}
\end{align}
We refer to the independent learning algorithm employing query policy \qgreedy as \naive.

\subsection{Exploiting Structural Constraints: \learnhm} \label{sec.algoPart1-Ideas.our}
We now present our algorithm $\learnhm$ in Algorithm~\ref{alg.learnhm}\footnote{Algorithm~\ref{alg.learnhm} reduces to \naive if $\qclique$ is replaced by $\qgreedy$ and $\luproj(\widetilde{L}^{t}, \widetilde{U}^{t})$ simply returns $\widetilde{L}^{t}, \widetilde{U}^{t}$.}. Our algorithm depends on three functions $\luproj$, $\qclique$ and $\getuserresponse$. A high-level description of these functions is given as follows:
\\
{\bf $\luproj$} shrinks the search space of hemimetrics by updating the lower and upper bounds from $\widetilde{L}^t, \widetilde{U}^t$ to $L^t, U^t$. Details are provided in Section~\ref{sec.algoPart2-LUBounds}. 
\\
{\bf $\qclique$} is the query policy that determines the next query $\querypoint^t = (i^t, j^t, \cost^t)$ at iteration $t$ given the current state of the learning process as determined by $L^{t-1}$ and $U^{t-1}$. Details are provided in Section~\ref{sec.algoPart3-Query}.
\\
{\bf $\getuserresponse$} returns the label $\labelpoint^t$ for query $\querypoint^t$. In the deterministic noise-free setting, this label is determined by Equation~\ref{eq.querylabel.D}. We also develop a robust noise-tolerant variant of $\getuserresponse$ for the stochastic setting, discussed in Appendix~\ref{app.algoPart4-Extensions-Noise}. 

\begin{algorithm}[!t]
	\caption{Our Algorithm: \learnhm}\label{alg.learnhm}
\begin{algorithmic}[1]
	\STATE{{\bfseries Input:} \looseness -1 set $\itemset$ of $\num$ items, range $\range$, error parameters $(\epsilon, \delta)$}
	\STATE{{\bfseries Output:} hemimetric $\Dhat$}
	\STATE{{\bfseries Initialize:} \\
		iteration $t = 0$; labeled data $\dataset^t = \emptyset$ \\
    	 	lower bounds: $L^t_{i, j} = 0 \ \forall \ i, j \in [\num]$ \\
    	   	upper bounds: $U^t_{i, j} = \range \ \forall \ i, j \in [\num]$ \\ 
    	   	\qquad   \qquad    \qquad $U^t_{i, i} = 0 \ \forall \ i \in  [n]$
    	   	}
	\WHILE{$\exists i,j\colon (U^t_{i,j} - L^t_{i,j}) > \epsilon $}
		\STATE{$t = t + 1$}	
   		\STATE{$\querypoint^t = (i^t, j^t, \cost^t) \gets \qclique(L^{t-1}, U^{t-1})$} 
   		\STATE{$\datapoint^t = ((i^t, j^t, \overline{\cost}^t), \labelpoint^t)  \gets  \getuserresponse(\querypoint^t)$ \mbox{\; \; \; \; \; \; \;}\emph{// $\overline{\cost}^t=\cost^t$ in the noise-free setting}}
   		\STATE{$\widetilde{U}^{t} = U^{t-1}, \widetilde{L}^{t} = L^{t-1}$}
    		\IF {$\labelpoint^t = 1$}
	   		\STATE update $\widetilde{U}^{t}_{i^t, j^t} = \overline{\cost}^t$
	   	\ELSE
	  		\STATE update $\widetilde{L}^{t}_{i^t,j^t} = \overline{\cost}^t$
		\ENDIF   		
   		\STATE{$L^{t}, U^{t} \gets  \luproj(\widetilde{L}^{t}, \widetilde{U}^{t})$}
   		\STATE{$\dataset^{t}  = \dataset^{t-1} \cup \{\datapoint^t\}$}
   	\ENDWHILE
   	\STATE{$\Dhat \gets U^t$}
 	\STATE{{\bfseries Return:} hemimetric $\Dhat$}
\end{algorithmic}
\end{algorithm}
\vspace{1mm}
\section{\luproj: Updating Bounds}\label{sec.algoPart2-LUBounds}
We now present the details of our function \luproj. The proof of Theorem~\ref{thm.projection} is given in Appendix~\ref{app.thm.projection} and the proof of Theorem~\ref{thm.luprojalgo} is given in Appendix \ref{app.thm.luprojalgo}. 
\subsection{Valid Bounds}
We begin by defining minimal conditions for the lower and upper bounds returned by \luproj to be \emph{valid} in terms of the version space. Let us start by formally defining the version space for our setting.
In Algorithm~\ref{alg.learnhm}, the labeled data at iteration $t$ is given by $\dataset^t = \{\datapoint^1, \ldots, \datapoint^t\}$ where $\datapoint^l = (\querypoint^l, \labelpoint^l)$ for $l \in [t]$. Then, the version space at iteration $t$ is defined as 
\vspace{-1mm}
\begin{align}
\hypclass^t   \coloneqq \{\hyp \in \hypclass \mid \forall \ l \in [t]\colon \hyp(\querypoint^l) = \labelpoint^l \},
\vspace{-1mm}
\end{align}
where $\hyp(x^l) = \indfunc( c^l \geq \hyp_{i^l,j^l} )$; here $\indfunc(\cdot)$ denotes the indicator function. That is, $\hypclass^t \subseteq \hypclass$ is the set of hemimetrics  at iteration $t$ that are consistent with the labeled data $\dataset^t$. \\
Also, for given lower bounds $L$ and upper bounds $U$, we define the set of hemimetrics satisfying these bounds as 
\vspace{-1mm}
\begin{align}
\hypclass(L, U)  \coloneqq \{\hyp \in \hypclass \ | \  L \leq  \hyp \leq  U\},
\vspace{-1mm}
\end{align}
where the inequalities are understood component-wise, {\em i.e.},  $L_{i, j} \leq  \hyp_{i, j} \leq  U_{i, j} \; \forall \ i, j \in [\num]$.
The bounds $L^t,U^t$ at iteration $t$ are \emph{valid}, iff $\hypclass(L^t, U^t) \supseteq  \hypclass^t$. Validity of the bounds ensures that $\hypstar$ is always contained in $\hypclass(L^t, U^t)$.

\subsection{Updating Bounds via Projection}
We formalize the problem of obtaining the bounds $L^t,U^t$ as the solution of the following optimization problem:
\begin{align}
  \min_{U, L} \quad &\| U - L \|_{\normsize} \tag{P1} \label{eq.jointprojection}\\
   \textnormal{s.t.} \quad & \hypclass(L, U) \supseteq \hypclass^t, \nonumber 
\end{align}
where the entry-wise $\ell_{\normsize}$-norm of a matrix is defined as $\|M\|_1 = \sum_{i,j} |M_{i,j}|$.
The intuitive idea behind this problem is to decrease the gap between upper and lower bounds as much as possible while ensuring that the resulting bounds are valid.\\
It turns out, \emph{cf.} Theorem~\ref{thm.projection}, that Problem~\ref{eq.jointprojection} can be solved in a two step process by solving the following two problems:
\begin{align}
&U^{*t} = \argmin_{U \in \ \uclass \text{ s.t. } U \leq  \widetilde{U}^{t}} \norm{U - \widetilde{U}^{t}}_{\normsize}  \tag{P2} \label{sec.algoPart2-LUBounds.proj.5} \\
&L^{*t} = \argmin_{L \in \ \lclass(U^{*t}) \text{ s.t. } L \geq  \widetilde{L}^{t}} \norm{L - \widetilde{L}^{t}}_{\normsize} \tag{P3} \label{sec.algoPart2-LUBounds.proj.6},
\end{align}
where $\widetilde{L}^t, \widetilde{U}^t$ are obtained by Algorithm~\ref{alg.learnhm} in lines 8--12.
Here, the set $\lclass$ parameterized by the upper bound matrix $U \in \uclass$ is defined as\vspace{2mm}
\begin{align}
&\!\!\lclass(U) := \{L \in \R^{\num \times \num} \mid L_{i, i} = 0, \; 0 \leq L_{i,j} \leq U_{i,j}, \;\;\; \label{sec.algoPart2-LUBounds.proj.2} \\					   	
						& \; L_{i, j} \geq \max \left( L_{i, k} - U_{j, k}, L_{k, j} - U_{k, i} \right) \;\; \forall \ i, j, k \in [\num]  \}. \notag
\end{align}\\[-2mm]
Additional details and a formal development of these sets are given in Appendix~\ref{app.classLU}. 
While the set of upper bound matrices corresponds to the set $\hypclass$ of bounded hemimetrics, the set $\lclass(U)$ represents more complex dependencies. It turns out that the set $\lclass$ cannot be constrained to contain only hemimetrics. In fact, we provide a counter-example in Appendix~\ref{app.classLU}. 
We can now state one of our main theoretical results:
\begin{theorem} \label{thm.projection}
The optimal solution of Problem~\ref{eq.jointprojection} is unique and is given by $L^{*t}, U^{*t}$ (defined in Problems~\ref{sec.algoPart2-LUBounds.proj.6} and ~\ref{sec.algoPart2-LUBounds.proj.5}).
\end{theorem}


\begin{figure}[t!]
\centering
\includegraphics[width=\linewidth]{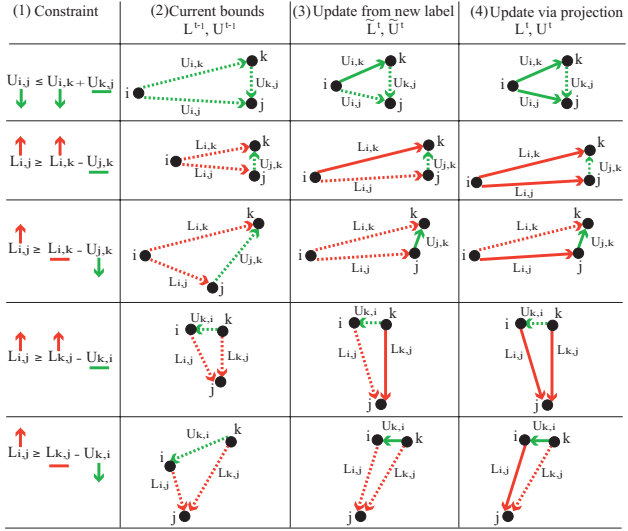}
\vspace{-7mm}
\caption{Geometric interpretation of the effect of exploiting the constraints in \lproj and \uproj ---
each row illustrates this for a particular constraint from the definitions of the sets $\uclass$ (Equation~\ref{eq.triangle.inequality}) and $\lclass$ (Equation~\ref{sec.algoPart2-LUBounds.proj.2}). Column $(1)$ shows the constraint, column $(2)$ shows the current lower and upper bounds ($L^{t-1}, U^{t-1}$ in Algorithm~\ref{alg.learnhm}), column $(3)$ shows an update from the new labeled datapoint ($\widetilde{L}^{t}, \widetilde{U}^{t}$ in Algorithm~\ref{alg.learnhm}), and column $(4)$ shows the effect of the corresponding constraint ($L^{t}, U^{t}$ in Algorithm~\ref{alg.learnhm}). For instance, in row 3, we consider the constraint $L_{i,j} \geq L_{i,k} - U_{j,k}$; after receiving a new label, $U_{j,k}$ decreases; this in turn can lead to an increase of $L_{i,j}$.} 
\label{fig.fig-lubounds}
\end{figure}

\subsection{Function \luproj: Tightening Bounds}

We now present an efficient solver for the optimization Problem~\ref{eq.jointprojection} given by the function \luproj in Algorithm~\ref{alg.luproj}. The algorithm is invoked with inputs $\widetilde{L}^{t}$, $\widetilde{U}^{t}$ --- its optimality is ensured by the following theorem. 

\begin{theorem} \label{thm.luprojalgo}
The lower and upper bounds $L^t, U^t$ returned by \luproj is the unique optimal solution of Problem~\ref{eq.jointprojection}. 
\end{theorem}

%
\begin{algorithm}[!t]
	\caption{Updating Lower \& Upper Bounds: \luproj}\label{alg.luproj}
\begin{algorithmic}[1]
	\STATE{{\bfseries Input:} \looseness -1 $\widetilde{L}^{t}, \widetilde{U}^{t}$; {\bfseries Output:} $L^{t}, U^{t}$}
   	\STATE{$U^{t} \gets  \texttt{\uproj}(\widetilde{U}^{t})$}
   	\STATE{$L^{t} \gets  \texttt{\lproj}(\widetilde{L}^{t}, U^{t})$}
   	\STATE{{\bfseries Return:} $L^{t}, U^{t}$}
\end{algorithmic}
\end{algorithm}
%
\begin{algorithm}[!t]
	\caption{Updating Upper Bounds: \uproj}\label{alg.uproj}
\begin{algorithmic}[1]
	\STATE{{\bfseries Input:} \looseness -1 $\widetilde{U}^{t}$; {\bfseries Output:} $U^{t}$}
  	 \STATE{{\bfseries Initialize:} $U^{t} =  \widetilde{U}^{t}$}
	\FOR{$k=1$ {\bfseries to}  $\num$}
		\FOR{$i=1$ {\bfseries to} $\num$}
				\FOR{$j=1$ {\bfseries to} $\num$}
					\STATE{ $U^{t}_{i, j} = \min {\left(U^{t}_{i, j}, U^{t}_{i, k} + U^{t}_{k, j}\right)}$ }
				\ENDFOR 
		\ENDFOR
	\ENDFOR
   	\STATE{{\bfseries Return:} $U^{t}$}
\end{algorithmic}
\end{algorithm}
%
\begin{algorithm}[!t]
	\caption{Updating Lower Bounds: \lproj}\label{alg.lproj}
\begin{algorithmic}[1]
	\STATE{{\bfseries Input:} \looseness -1 $\widetilde{L}^{t}, U^{t}$; {\bfseries Output:} $L^{t}$}
  	 \STATE{{\bfseries Initialize:} $L^{t} =  \widetilde{L}^{t}$}
	\FOR{$k=1$ {\bfseries to}  $\num$}
		\FOR{$i=1$ {\bfseries to} $\num$}
				\FOR{$j=1$ {\bfseries to} $\num$}
					\STATE{ $L^{t}_{i, j} = \max {\left(L^{t}_{i, j}, L^{t}_{i, k} - U^{t}_{j, k}, L^{t}_{k, j} - U^{t}_{k, i}\right)}$ }			
				\ENDFOR 
		\ENDFOR
	\ENDFOR
   	\STATE{{\bfseries Return:} $L^{t}$}
\end{algorithmic}
\end{algorithm}

We now briefly describe the function \luproj. It first invokes the function \uproj shown in Algorithm~\ref{alg.uproj} to compute $U^t$ --- which in fact equals the optimal solution of Problem~\ref{sec.algoPart2-LUBounds.proj.5} (refer to the proof of the theorem). Then, it invokes the function \lproj  shown in Algorithm~\ref{alg.lproj} to compute $L^t$ --- which in fact equals the optimal solution of Problem~\ref{sec.algoPart2-LUBounds.proj.6} (again, refer to the proof of the theorem).  Both $\uproj$ and $\lproj$ can be seen as iterating over the constraints of the sets $\uclass$ and $\lclass$, updating variables that violate constraints.  $\uproj$ is in fact equivalent to the Floyd-Warshall algorithm for solving the all-pair shortest paths problem in a graph~\cite{floyd1962algorithm}. 
Similar equivalence has been shown by~\citet{brickell2008metric} while studying the problem of projecting a non-metric matrix to a metric via decrease-only projections. 
The function $\lproj$ operates in similar fashion as $\uproj$. However, additional challenges in the interpretation and analysis of the solution of $\lproj$ arise from the fact that the class $\lclass$ is not a set of hemimetrics and has more complex dependencies.\\
Figure~\ref{fig.fig-lubounds} provides a geometric interpretation of the constraints imposed by the sets $\uclass$ (Equation~\ref{eq.triangle.inequality}) and $\lclass$ (Equation~\ref{sec.algoPart2-LUBounds.proj.2}) exploited by \uproj and \lproj in line~6.

\subsection{Geometric Interpretation of $L^{*t}$ and $U^{*t}$}



In this section we provide a geometric interpretation of the optimal solution to Problem~\ref{eq.jointprojection}.
Consider the space $\R^{\num^2}$. Let us define $\pi^0$ to be the set of inequalities 
\begin{align*}
 \{ &\hyp_{i,i} = 0, 0 \leq \hyp_{i,j} \leq \range, \\
&\hyp_{i,j} \leq \hyp_{i,k} + \hyp_{k,j} \quad \forall \ i, j, k \in [n]\}.
\end{align*}
Thus, the subset of $\R^{\num^2}$ described by $\pi^0$ corresponds to the set of bounded hemimetrics.
At iteration $t$, we get a new labeled datapoint $\datapoint^t = ((i^t, j^t, \cost^t), \labelpoint^t )$ and update the set of inequalities as follows:
\begin{align*}
  \pi^{t} = \begin{cases}
       \pi^{t-1} \cup \{ \cost^t \geq \hyp_{i^t,j^t} \} & \text{ if } \labelpoint^t = 1, \\
       \pi^{t-1} \cup \{ \cost^t < \hyp_{i^t,j^t}\} & \text{ if } \labelpoint^t = 0.
   \end{cases}
\end{align*}
Now, consider the polytope $\Lambda^t$ defined by $\pi^t$ in $\R^{\num^2}$. Furthermore, consider the hypercube in $\R^{\num^2}$ described by any lower and upper bounds $L,U$ --- the extent of the hypercube in dimension $(i,j)$ is given by $[L_{i,j}, U_{i,j}]$.
Then, the optimal solution to Problem~\ref{eq.jointprojection} describes the {\em unique} tightest hypercube containing the polytope $\Lambda^t$.
\section{\qclique: Proposing Queries}\label{sec.algoPart3-Query}

We first show the limitations of a greedy myopic policy \qgreedy for proposing queries, and then design our query policy \qclique that overcomes these limitations. 

\subsection{Myopic Policy \qgreedy}

The policy \qgreedy is inspired by the idea of shrinking the  gap $(U^{t-1}_{i,j} - L^{t-1}_{i,j})$ of pair $(i^t, j^t)$ with maximum uncertainty. As already mentioned, this policy can be seen as myopic (greedy) in terms of minimizing the objective $\| \hyphat - \hypstar \|_\infty$. However, this policy may turn out to be suboptimal in terms of exploiting the structural constraints of hemimetrics.\\
In particular, consider a simple instance with $n$ items belonging to a single tight cluster, \emph{i.e.}, $\forall i, j \in [\num]\colon \hypstar_{i,j} = 0$. Clearly, given the $2n$ distances $\hypstar_{1,j}, \hypstar_{j,1}$ for $j \in [\num]$, one can infer all other distances because of the tightness of the triangle inequalities.
For this example, \naive will, in every iteration $t$, half the largest upper bound $\max_{i,j} U^t_{i,j}$. Thus, in iteration $t$, all upper bounds $U^t_{i,j}$ are in the range $[\sfrac{\alpha}{2}, \alpha]$ where $\alpha = \range \cdot 2^{-\lfloor \sfrac{t}{(\num^2 - \num)}\rfloor }$ and all lower bounds $L^t_{i,j}=0$. Here, invoking \luproj 
cannot exploit the structural constraints, \emph{i.e.}, it would simply return $(\widetilde{L}^t, \widetilde{U}^t)$.

%


\subsection{Non-myopic Policy \qclique}

We now present an alternative policy \qclique that proposes pairs $(i^t,j^t)$ in a non-myopic way in Algorithm~\ref{alg.qclique}. Instead of greedily minimizing $\| \hyphat - \hypstar \|_\infty$, we aim to learn distances in a systematic way such that we can exploit the structural constraints of the hemimetrics more effectively. Note that in this section, we assume a fixed ordering of the items indexed by $1, \ldots, \num$.\\
The high level idea behind \qclique is to maintain a clique of items for which all pairwise distances are already learnt. It proposes queries in a systematic way to grow the clique item by item according to the assumed ordering. 
In more detail, the policy works as follows:
\begin{enumerate}
  \item At iteration $t$, the policy maintains a clique $\mathcal{C} =\{ 1, \ldots, |\mathcal{C}| \} \subseteq \itemset$ of items --- see line 2 of Algorithm~\ref{alg.qclique}. For any pair of items $(i,j) \in \mathcal{C}$ it holds that $ U^{t-1}_{i,j} - L^{t-1}_{i,j} \leq \epsilon$. 
  \item It then identifies the next item to be added to the clique, denoted as $a$ in line 3. 
  \item Now, it picks an item $b \in \mathcal{C}$ for which the distance to $a$ is not learnt up to precision $\epsilon$ in line 4, and returns the next query. 
\end{enumerate}

{\bf Online model.}
%
In many practical scenarios, the $\num$ items are not known beforehand, and rather appear over time. Let us denote by $\itemset^\Time = \{ 1, \ldots, \Time\} \subseteq \itemset$ the set of items present at time $\Time$. When a new item arrives at time $\Time + 1$ we could learn a hemimetric solution from scratch. However, it would be desirable to make use of the hemimetric solution for $\itemset^\Time$ and extend it.
In fact, \learnhm together with \qclique can be readily applied to this scenario. We can identify the clique $\mathcal{C}$ with the itemset $\itemset^{\Time}$, where $\Time = |\mathcal{C}|$, for which the hemimetric solution is known. The idea of adding $a$ to $\mathcal{C}$ in Algorithm~\ref{alg.qclique}, is then equivalent to extending the hemimetric solution for $\itemset^\Time$  to $\itemset^{\Time+1}$. By this equivalence, the sample complexity of growing the hemimetric solution item by item up to size $\num$ is the same as that of computing the hemimetric solution for all $\num$ items at once.

%
%

\begin{algorithm}[!t]
	\caption{Query Policy: \qclique}\label{alg.qclique}
\begin{algorithmic}[1]
	\STATE{{\bfseries Input:} \looseness -1 $L^{t-1}, U^{t-1}$; {\bfseries Output:} query $(i^t, j^t, c^t)$}
	\STATE{$\mathcal{C} = \{ i | \forall j < i\colon U^{t-1}_{i,j} - L^{t-1}_{i,j} \leq \epsilon \ \wedge \ U^{t-1}_{j,i} - L^{t-1}_{j,i} \leq \epsilon \}$}
    \STATE{$a = \max \mathcal{C} + 1$}
    \STATE{$b = \text{minimal } i \in \mathcal{C} \text{ s.t. } (U^{t-1}_{a,i} - L^{t-1}_{a,i} > \epsilon \ \vee \ U^{t-1}_{i,a} - L^{t-1}_{i,a} > \epsilon )$}
	\IF{$U^{t-1}_{a,b} - L^{t-1}_{a,b} > \epsilon$}
	  \STATE{{\bf Return:} $(a, b, \tfrac{1}{2} ( L_{a,b} + U_{a,b}))$}
	\ELSE
	  \STATE{{\bf Return:} $(b, a, \tfrac{1}{2} ( L_{b,a} + U_{b,a}))$}
	\ENDIF
\end{algorithmic}
\end{algorithm}
\vspace{3mm}
\section{Performance Analysis}\label{sec.algoPart4-Performance}

In this section we analyze the sample complexity and runtime of our proposed algorithm \learnhm. 
All proofs are provided in Appendix~\ref{app.algoPart3-Query}. 

\subsection{Sample Complexity}
Motivated by our preference elicitation application (see Section~\ref{sec.experiments}), we analyze sample complexity under a clustered-ness assumption. In particular, we say the hypothesis $\hypstar$ is $(\rangeintra,\clusters)$-clustered, if the following condition holds:
The items are partitioned into $\clusters$ clusters, such that for any pair of items $(i,j)$ their distance is $\hypstar_{i,j} \in [0, \rangeintra]$ if $i$ and $j$ are from the same cluster and $\hypstar_{i,j} \in [\rangeintra, \range]$ otherwise. 
%
Note that $\clusters$ and $\rangeintra$ are \emph{unknown} to the algorithm.
For this setting, the sample complexity of \learnhm is bounded by the following theorem.

\begin{theorem} \label{theorem.complexity1}
  If $\hypstar$ is $(\rangeintra,\clusters)$-clustered, the sample complexity of \learnhm is upper bounded by 
\begin{align*}
  2 \num  \clusters \Big\lceil \log\Big(\frac{\range}{\epsilon}\Big) \Big\rceil + \num^2 \Big\lceil \log\Big(\frac{2\rangeintra + 3\epsilon}{\epsilon}\Big) \Big\rceil.
\end{align*}
\end{theorem}

In real-world applications, the distances $\hypstar_{i,j}$ might correspond to monetary incentives and are, therefore, naturally quantized to some precision $\Delta$ (monetary incentives are multiples of the smallest currency unit, \emph{e.g.}, one cent). In this setting, the learning algorithms can learn $\hypstar$ exactly, \emph{i.e.}, $\hyphat = \hypstar$, with a bounded number of queries. The idea is that both, \naive and \learnhm, can collapse the gap $U^t_{i,j} - L^t_{i,j}$ to zero whenever $U^t_{i,j} - L^t_{i,j} < \Delta$. Hence, by invoking these algorithms with any $\epsilon < \Delta$, we learn $\hypstar$ exactly. We then obtain the following corollary for this interesting special case.

\begin{corollary} \label{theorem.complexity2}
If $\hypstar$ is $(0,\clusters)$-clustered, and assuming all distances $\hypstar_{i,j}$ are quantized to precision $\Delta > 0$, the sample complexity of \learnhm to exactly learn $\hypstar$ is upper bounded by $2 \num  \clusters  \big\lceil \log\big(\frac{\range}{\Delta}\big) \big\rceil$.
%
This matches the lower bound of \ $\Omega(\num \clusters)$.
\end{corollary}
%
%
\vspace{-0.1cm}
Note that our algorithm \learnhm does not perform more queries than \naive for any instance. In fact, the hardest instance for our algorithm is given by $\hypstar$ where all distances $\hypstar_{i,j} = \sfrac{r}{2}$. In this case, \luproj cannot exploit any structural constraints --- the number of queries performed by \learnhm exactly equals that of \naive (equal to $\num^2 \lceil \log( \tfrac{\range}{\epsilon} )\rceil $).


{\bf Stochastic responses.} We can also bound the sample complexity for the more realistic case in which query responses are stochastic. 
Here, we briefly introduce our noise model --- a more detailed description is given in Appendix~\ref{app.algoPart4-Extensions-Noise}. 
Our noise model is parametrized by the variance matrix $\sigma \in \R_+^{\num,\num}$ \emph{unknown} to the algorithm. The acceptance function $\prob(\labelfunc((i,j,\cost))=1)$ is given by the CDF of a normal distribution $\mathcal{N}(\hypstar_{i,j}, \sigma_{i,j})$ truncated to $[\hypstar_{i,j} - \beta_{i,j}, \hypstar_{i,j} +  \beta_{i,j}]$ where $\beta_{i,j} = \min \{ \hypstar_{i,j}, \range - \hypstar_{i,j} \}$.
Note that in our model the noise is unbounded, \emph{i.e.}, at $c=\hypstar_{i,j}$ the acceptance function $\prob(\labelfunc((i,j,\cost))=1)=0.5$. In order to deal with this, we develop a robust noise-tolerant
variant of \getuserresponse which ensures that the maximum noise experienced by the algorithm is bounded by
\begin{align*}
  \eta_{max} = \frac{1}{2} - \frac{\int_{0}^{\sfrac{\epsilon}{3}} e^{-\frac{w^2}{2\max_{i,j} \sigma_{i,j}^2}} \text{d}w }{  \int_{-\sfrac{\range}{2}}^{\sfrac{\range}{2}} e^{-\frac{w^2}{2\max_{i,j} \sigma_{i,j}^2}} \text{d}w }.
\end{align*}\\[1mm]
The sample complexity bounds are characterized by the quantity $\gamma = \big( \frac{3 \ln(\sfrac{3\num^2}{\delta})}{(0.5 - \eta_{\max} )^2} \big)$ --- the theoretical results corresponding to the settings in Theorem~\ref{theorem.complexity1} and Corollary~\ref{theorem.complexity2} are given as follows.
\begin{theorem} \label{theorem.complexity3}
  With probability $1-\delta$, \learnhm learns $\hypstar$ with precision $\epsilon$ and the sample complexity is upper bounded by\footnote{The $\widetilde{\bigObound}(\cdot)$ notation is used to omit factors logarithmic in the factors present explicitly.} 
\begin{align*}
   \widetilde{\bigObound} \Big( \gamma \Big( 2\num  \clusters \Big\lceil \log\Big(\frac{3 \range}{\epsilon}\Big) \Big\rceil + \num^2 \Big\lceil \log\Big(\frac{6\rangeintra + 9\epsilon}{\epsilon}\Big) \Big\rceil \Big) \Big).
\end{align*}
\end{theorem}

%
\begin{corollary} \label{theorem.complexity4}
Consider the case $\hypstar$ is $(0,\clusters)$-clustered, and assume all distances $\hypstar_{i,j}$ are quantized to precision $\Delta > 0$.  With probability $1-\delta$, \learnhm learns $\hypstar$ exactly and the sample complexity is upper bounded by 
\begin{align*}
  \widetilde{\bigObound} \Big( \gamma 2\num  \clusters  \Big\lceil \log\Big(\frac{3\range}{\Delta}\Big) \Big\rceil \Big).
\end{align*}
\end{corollary}
%


\subsection{Runtime Analysis and Speeding Up \learnhm} \label{sec.algoPart4-Performance.speedup}


We now begin by analyzing the runtime of \learnhm. The algorithm $\learnhm$ invokes $\luproj$ after every query --- there are  $\Thetabound(\num^2 \log(\tfrac{\range}{\epsilon}))$ calls to $\luproj$ in the worst case. 
The runtime of $\luproj$ is $\Thetabound(\num^3)$, resulting in a total runtime $\Thetabound (\num^5 \log(\tfrac{\range}{\epsilon} ))$ --- this is prohibitively expensive for most realistic problem instances.\\
The key idea for speeding up \learnhm is that we can choose the constraints that should be exploited instead of exploiting all the constraints after every query (line 6 in \uproj \& \lproj). If the constraints to be exploited are selected carefully, we can still get reasonable benefits from tightening lower and upper bounds. For \learnhm, this can be achieved as follows:
\vspace{-1mm}
\begin{itemize}
\item First, we do not need to invoke \luproj after every query (this is equivalent to not exploiting any violated constraint). At any iteration $t$, \learnhm invokes $\luproj$ only if $\widetilde{U}^t_{i^t,j^t} - \widetilde{L}^t_{i^t,j^t} \leq \epsilon$. Otherwise, it simply sets $(L^t, U^t) \gets (\widetilde{L}^t, \widetilde{U}^t)$.
\item Second, when \learnhm invokes \luproj at iteration $t$ we only consider the $2\num$ constraints which involve $i^t,j^t$ in lines 4 and 5 of Algorithms~\ref{alg.uproj} and~\ref{alg.lproj}.
\end{itemize}
\vspace{-1mm}
Details of this idea are presented in Appendix~\ref{app.algoPart4-Extensions-Speedup}. 
\learnhm implementing this idea invokes $\luproj$ at most $\num^2$ times each with a runtime of $4 \num$. Hence, the total runtime of the speeded up \learnhm is $\Thetabound(\num^3)$. Most importantly, the sample complexity bounds from the previous section still apply.

\section{Experimental Evaluation}\label{sec.experiments}


\subsection{Benchmarks}

We compare the performance of the speeded up \learnhm against the baseline \naive. We also compare against a second baseline \naiveSIT (\naive with side information of triplet comparisons). 
This baseline utilizes a low-dimensional embedding of the items as a preprocessing step. Following the work of \citet{jamieson2011low}, using $\num^2 \log \num $ triplet queries, \emph{i.e.}, for a triplet $(i, j, k)$ such a query returns $\indfunc(\hypstar_{i,j} \leq\hypstar_{i,k})$, one can compute a low-dimensional embedding of the items. Using this embedding, one can infer the response to all possible $\num^3$ triplet queries --- this is the side information that we supply to \naiveSIT. \\
This side information can be exploited by \naiveSIT as follows. For a given triplet, $\indfunc(\hypstar_{i,j} \leq \hypstar_{i,k})=1$ implies two constraints on the lower and the upper bounds --- $U_{i,j} \leq U_{i,k}$ and $L_{i,k} \geq L_{i,j}$. After every query, \naiveSIT first updates an upper or lower bound according to the response. Then it exploits these two constraints for $\num^3$ triplets to tighten the bounds on every pair of items. We will report the sample complexity of \naiveSIT as the total number of queries for computing the low-dimensional embedding and for learning all distances up to precision $\epsilon$.


\subsection{Experimental Setup}
 
\begin{figure*}[!t]
  \begin{subfigure}[b]{0.33\textwidth}
    \centering
    \includegraphics[width=0.92\textwidth]{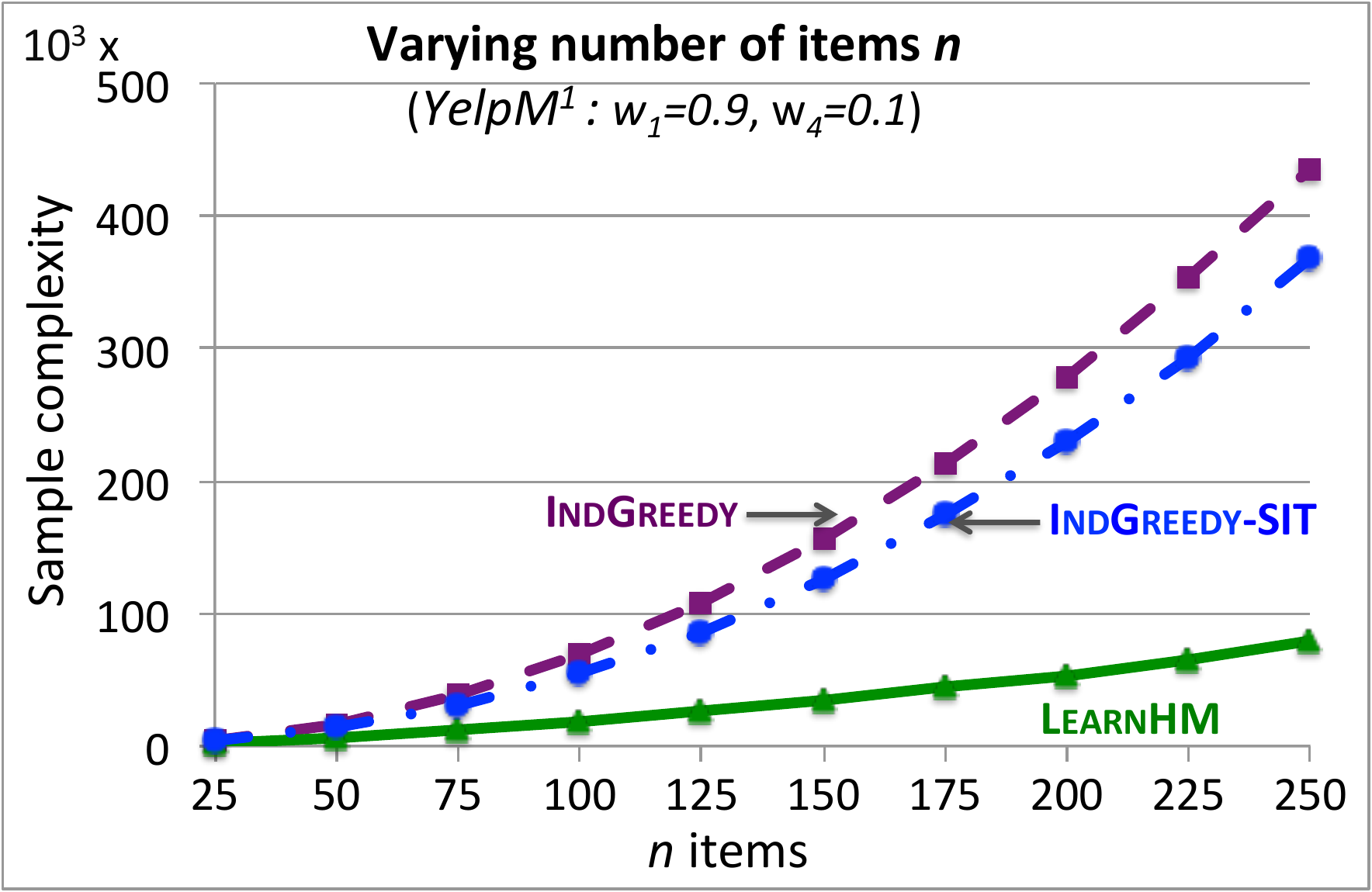}
      \vspace{-1mm}
    \caption{Varying $\num$ (\emph{YelpM}$^1$)}
    \label{fig.model1.vary.n}
  \end{subfigure}
  \begin{subfigure}[b]{0.33\textwidth}
    \centering
    \includegraphics[width=0.92\textwidth]{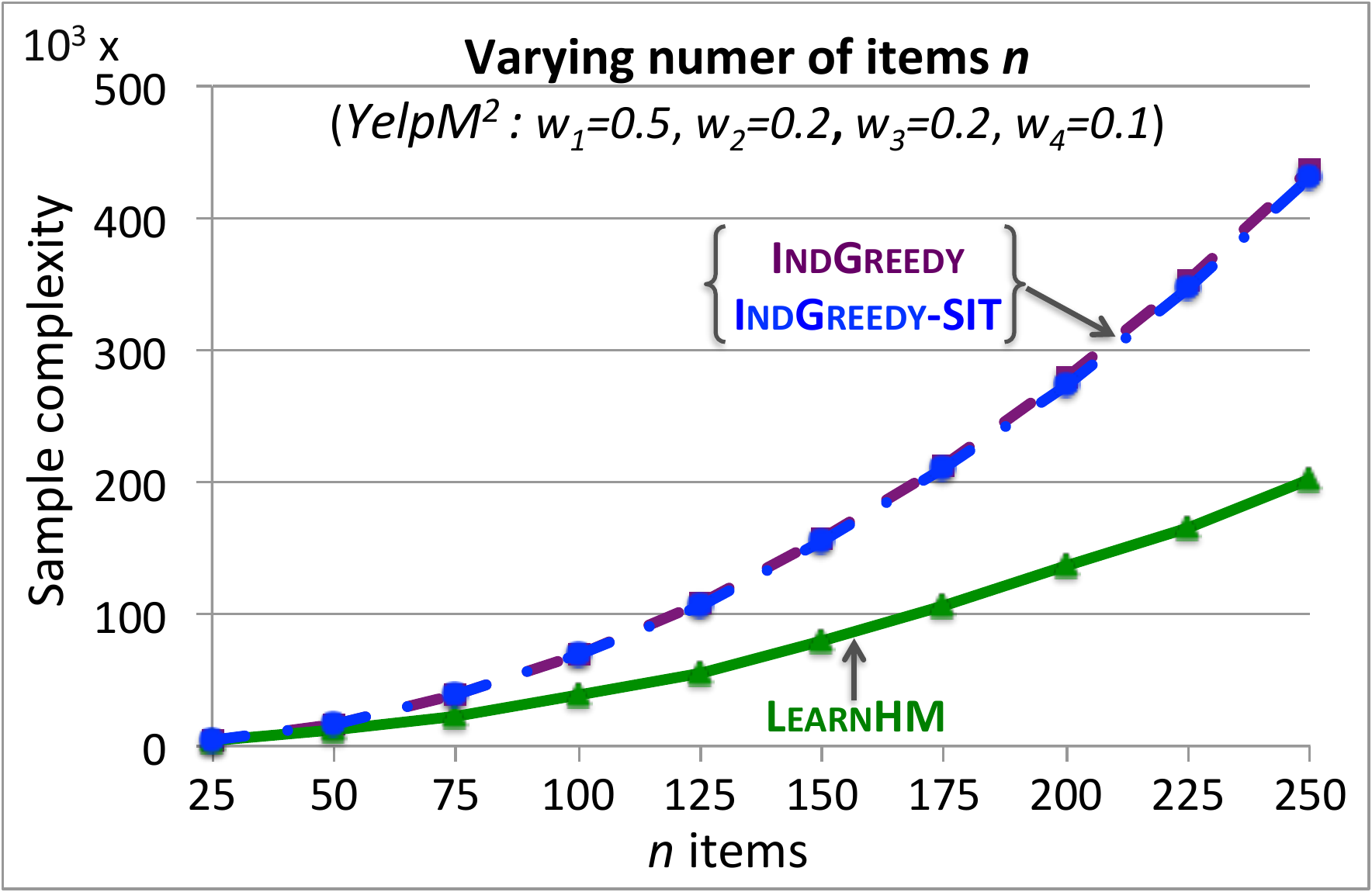}
      \vspace{-1mm}    
    \caption{Varying $\num$ (\emph{YelpM}$^2$)}
    \label{fig.model2.vary.n}
  \end{subfigure}
  \begin{subfigure}[b]{0.33\textwidth}
    \centering
    \includegraphics[width=0.92\textwidth]{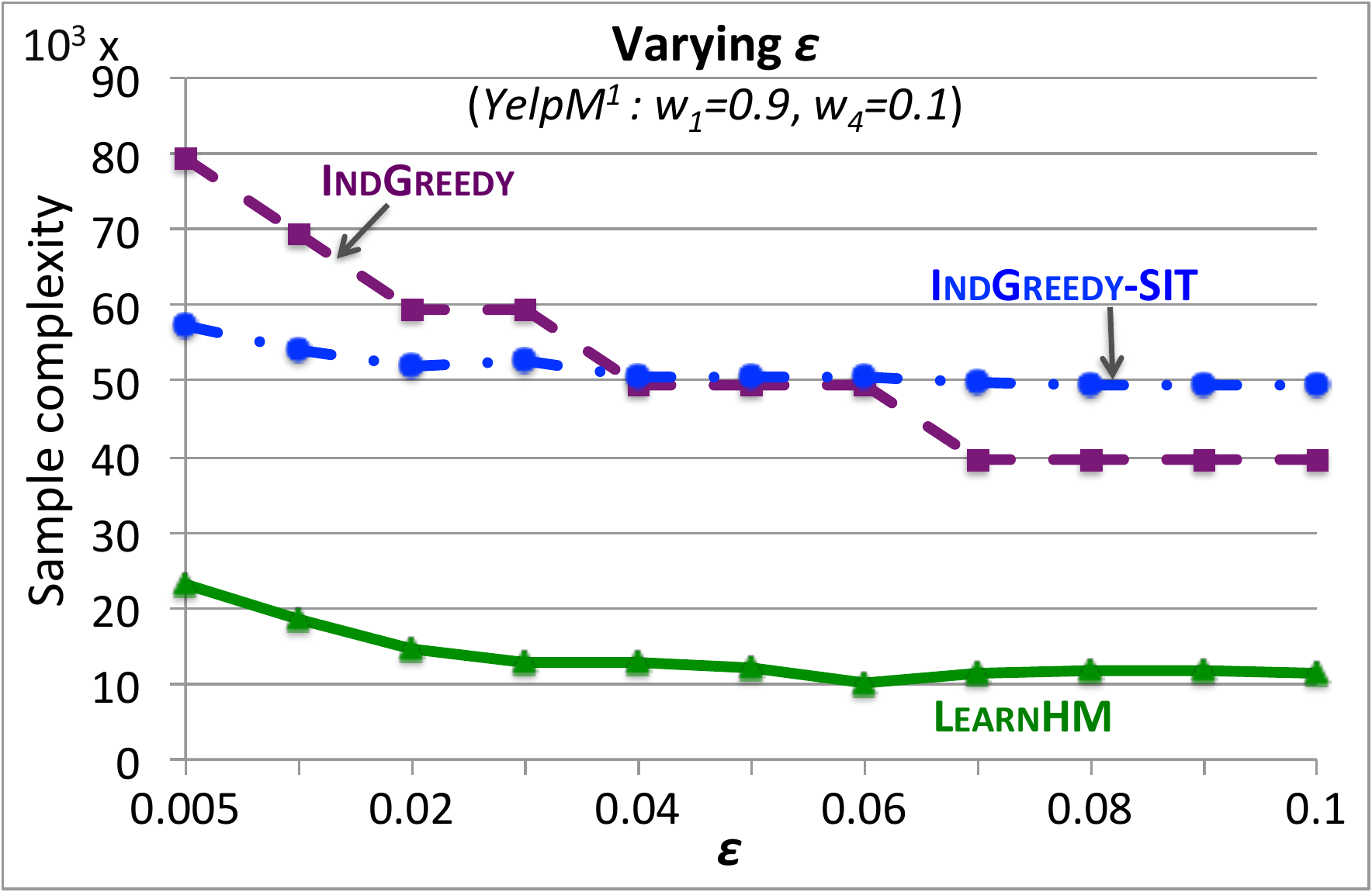}
      \vspace{-1mm}    
    \caption{Varying $\epsilon$ (\emph{YelpM}$^1$)}
    \label{fig.model1.vary.eps}
  \end{subfigure}
  \vspace{-9mm}
  \caption{Sample complexity results for two different user preference models defined via the \emph{Yelp} dataset. For the first model corresponding to the $\clusters=5$ cluster setting, \emph{cf.}\ Theorem~\ref{theorem.complexity1}, the sample complexity of \learnhm is about an order of magnitude smaller than that of \naive. Even for the generic model, \learnhm substantially outperforms the two baselines.}
  \label{fig.results}
\end{figure*}

{\bf Yelp dataset. }We use the recently proposed \emph{Yelp Dataset Challenge (round 7)} data for our experiments.\footnote{\texttt{https://www.yelp.com/dataset\_challenge/}} This data contains information about $77$K businesses  located across 10 cities around the world. We looked into businesses belonging to the category \emph{Restaurants} and being located in the city of \emph{Pittsburgh, PA}. In particular, we extracted information for all 290 restaurants offering food from the cuisines Mexican (50), Thai (26), Chinese (53), Mediterranean (75), Italian (86). For each of these restaurants we also collected the review count and coordinates (longitude and latitude). We discretized the review count into {\em High} (popular, 166 restaurants) when there were more than 25 reviews and into {\em Low} (unpopular, 124 restaurants) otherwise. The collected data is visualized in Figure~\ref{fig.dataset}. \\
{\bf User preference models.} We simulate user preference models from this data by creating the underlying hemimetric $\hypstar$ as follows. For notational ease, we use the shorthands $\cuisine_i$, $\review_i$, $\latitude_i$, $\longitude_i$ to refer to the above mentioned attributes for item $i$. We quantify the distance between item $i$ and $j$ by
\begin{align*}
  W_{i,j} = w_1 W^{\cuisine}_{i,j} + w_2 W^{\review}_{i,j} + w_3 W^{\geo}_{i,j} + w_4 W^{\rand}_{i,j},
\end{align*}
where
\begin{align*}
  W^{\cuisine}_{i,j} &= \range \indfunc(\cuisine_i \neq \cuisine_j), \\
  W^{\review}_{i,j} &=  \range \indfunc(\review_i > \review_j), 
\end{align*}
$W^{\geo}_{i,j}$ is the great-circle distance based on the latitude and longitude coordinates normalized to lie in $[0,\range]$, and $W^{\rand}_{i,j}$ is drawn uniformly at random from $[0,\range]$. Recall, $\range$ is the upper bound on the distance. The weights $w_1, \ldots, w_4 \in \R_+$ sum up to 1. We compute $\hypstar$ as the closest metric to $W$ according to~\citet{brickell2008metric}.\\
For different weights $w_1, \ldots, w_4$, we can instantiate different user preference models $\hypstar$. In particular, we instantiate the following two models. In the first model (\emph{YelpM}$^1$), we use $w_1 = 0.9, w_4=0.1$ and $w_2 = w_3 = 0$ --- this corresponds to the setting we considered in Theorem~\ref{theorem.complexity1} with $\clusters=5$ and intra cluster distance $\rangeintra \leq w_4 \cdot \range$. The second model (\emph{YelpM}$^2$) is more generic, with weights given by $w_1=0.5,w_2=w_3=0.2,w_4=0.1$.

\begin{figure}[h]
  \centering
  \includegraphics[width=0.9\columnwidth]{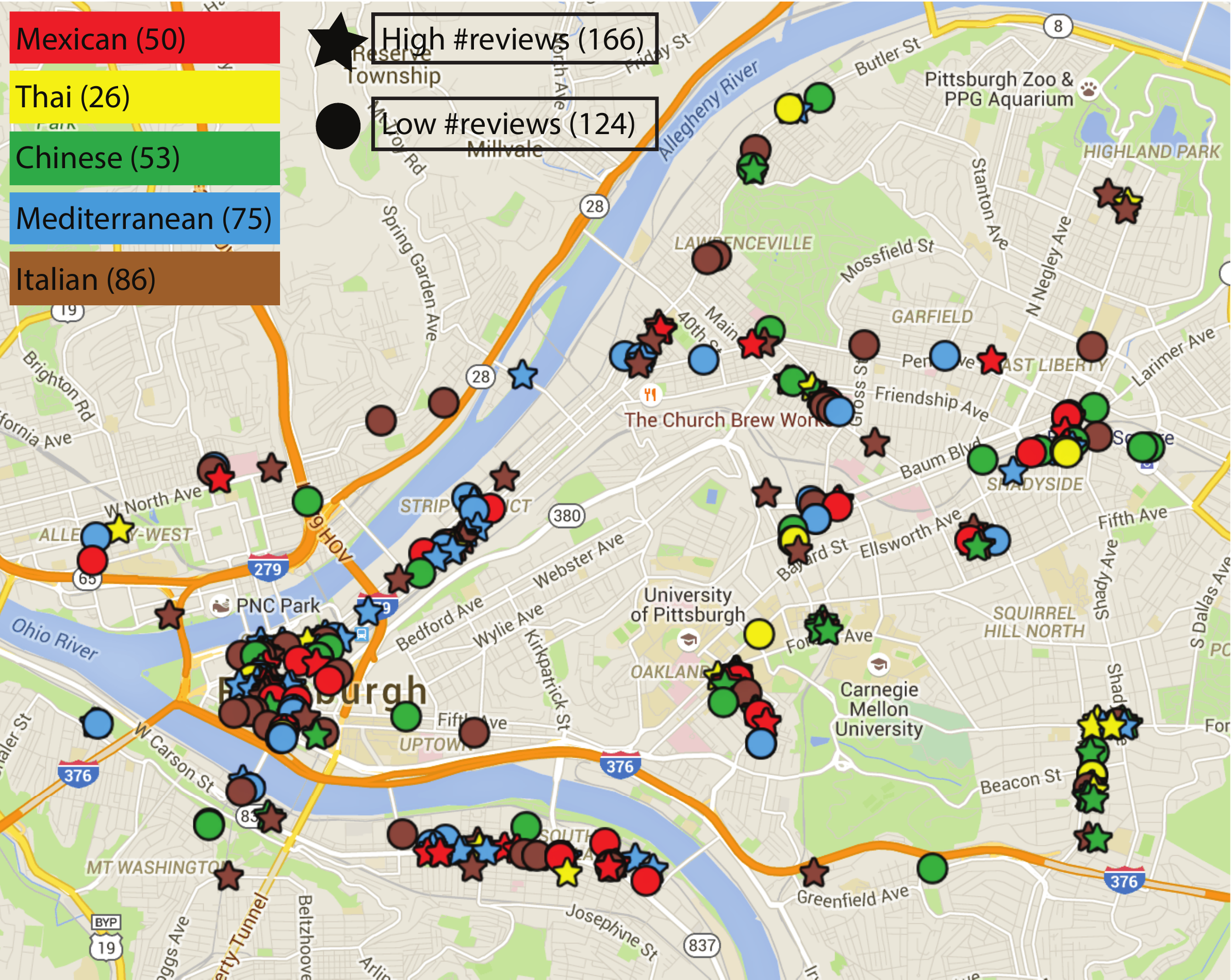}
  \vspace{-3mm}
  \caption{Visualization of the 290 restaurants from the \emph{Yelp} data in Pittsburgh. The restaurants are distinguished by cuisine (color), popularity (star) and location (coordinates).}
  \label{fig.dataset}
\end{figure}


\subsection{Results}

We  now present our results. We focus on the noise-free setting here, results for the stochastic setting are provided in Appendix~\ref{app.experiments}. 
As the metric for comparing \learnhm, \naive and \naiveSIT, we use the sample complexity of learning the unknown hemimetric $\hypstar$ up to precision $\epsilon$. In the experiments we used $\range=1$, and varied $\epsilon$ and $\num$. When varying $\num$ (by sub-sampling), we used $\epsilon=0.01$. Results are shown in Figures~\ref{fig.results} (a-c). \\
In Figure~\ref{fig.model1.vary.n}, for the first model \emph{YelpM}$^1$, we can observe that the sample complexity of \learnhm, \emph{cf.}\ Theorem~\ref{theorem.complexity1}, is almost an order of magnitude smaller than that of \naive. Note that \naiveSIT also has lower sample complexity than \naive. However, because of the large cost for computing an embedding, \naiveSIT performs worse than \learnhm.\\
The results for the more generic model \emph{YelpM}$^2$ without specific clusters are shown in Figure~\ref{fig.model2.vary.n}. We observe that \learnhm and \naiveSIT have higher sample complexity compared to the first model \emph{YelpM}$^1$. The sample complexity of \learnhm is still better by a factor of ${\sim}2$ than that of the baseline algorithms --- although the assumptions in Theorem~\ref{theorem.complexity1} do not hold.\\
Finally, Figure~\ref{fig.model1.vary.eps} shows results for varying $\epsilon$ for $\num=100$ items. As expected, with increasing $\epsilon$ the sample complexity of all algorithms increases. For larger values of $\epsilon$, \naiveSIT performs worse than \naive as the sample complexity for computing the embedding dominates the learning. Most importantly, \learnhm has the lowest sample complexity of all three algorithms.

%




\section{Related Work}\label{sec.related}

\subsection{Metric Learning \& Low-dimensional Embeddings}
Learning distances to capture notions of similarity or dissimilarity plays a central role in many machine learning applications.
We refer the reader to the detailed surveys by \citet{yang2006distance,van2009dimensionality,bellet2013survey}. Some of the key distinctions of our formulation from the existing  research are described in the following.\\
{\bf Supervised metric learning.}
In their seminal work, \citet{xing2002distance} introduced the supervised metric learning framework for learning Mahalanobis distance functions via a convex formulation. In contrast to our setting, they assume that the algorithm has access to the feature space of the input data. Furthermore, this framework and its variants are restricted to recover symmetric distance functions.
Another line of research considers learning asymmetric distances, for instance by learning local invariant Mahalanobis distances~\cite{fetaya2015learning} or by learning general bilinear similarity matrices~\cite{liu2015similarity}. However, this line of work is not directly applicable to our setting because it also requires access to the feature space of the input data.\\
{\bf Learning embeddings.} 
Another line of research is that of learning low-dimensional embeddings for a set of items respecting the observed geometric relations between these items~\cite{amid2015multiview,tamuz2011adaptively,cox2000multidimensional,jamieson2011low}. 
For applications involving human subjects, a triplet-based queries framework --- for a triplet $(i, j, k)$ it queries $\indfunc(\D_{i,j} \leq \D_{i,k})$ ---  has been employed.
However the distances recovered from these approaches are merely optimized to respect the observed relations seen in the data from query responses --- they are symmetric and importantly \emph{do not} have an actual quantitative (economic) interpretation as we seek in our formulation. 

\subsection{Exploiting Structural Constraints}
Our approach of exploiting the structural constraints of the hemimetrics polytope is in part inspired by \citet{elkan2003using}, who accelerates $\clusters$-means by exploiting the triangle inequality. By maintaining bounds on the distances, he efficiently reduces the number of distance computations. 
\citet{brickell2008metric} study the problem of projecting a non-metric matrix to a metric matrix, and 
consider a specific class of decrease-only projections.
Our approach towards updating upper bounds via decrease-only projections is similar in spirit. However, the main technical difficulties arise in maintaining and updating the lower bounds, for which we develop novel techniques.
The active-learning approach proposed by~\citet{jamieson2011low} exploits the geometry of the embedding space to minimize the sample complexity using triplet-based queries. While similar in spirit, their approach is based on triplet-based queries, and differs from our methodology of exploiting the structural constraints.

\subsection{Learning User Preferences}

Another relevant line of research is concerned with eliciting user preferences. 
Specifically, we seek to learn private costs of a user for switching from her default choice of item $i$ to instead choose item $j$. This type of preferences can be used in marketing applications, \emph{e.g.}, for persuading users to change their decisions~\cite{kamenica2009bayesian}. \citet{singla2015incentivizing} considered similar preferences in the context of balancing a bike-sharing system by incentivizing users to go to alternate stations for pickup or dropoff of bikes.  
\citet{abernethy2015low} considered the application of purchasing data from users, and quantified the prices that should be offered to them. One key difference of our approach is that we are interested in jointly learning the preferences between $\num$ items, \emph{i.e.}, tackling $\num^2$ learning problems jointly.

\subsection{Active Learning}

Our problem formulation shares the goal of reducing sample complexity with other instances of active learning~\cite{settles2012active}. Our goal to specifically exploit the structural constraints of the hemimetric polytope is along the lines of research in structured active learning where the hypotheses or the output label space have some inherent structure that can be utilized, \emph{e.g.}, the structure of part-of-speech tagging of the sentence~\cite{roth2006margin}.

\vspace{0.6mm}
\section{Conclusions}\label{sec.conclusions}


We investigated the novel problem of actively learning hemimetrics. 
The two key techniques used in the construction of our algorithm \learnhm are novel projection techniques for tightening the lower and upper bounds on the solution space and a non-myopic (non-greedy) query policy. Our algorithm can be readily applied to the online setting allowing one to extend the hemimetric solution over time. We provided a thorough analysis of the sample complexity and runtime of our algorithm. Our experiments on \emph{Yelp} data showed substantial improvements over baseline algorithms in line with our theoretical findings.

{\bf Acknowledgements.} This research is supported in part by SNSF grant 200020\_159557 and the Nano-Tera.ch program as part of the Opensense II project.

%

\bibliographystyle{icml2016}
\bibliography{arxiv-hemimetrics}  


\clearpage

\onecolumn
\appendix
{\allowdisplaybreaks


\section{Robust Noise-tolerant Algorithms}\label{app.algoPart4-Extensions-Noise}

In this section, we introduce our noise model and the corresponding stochastic acceptance function. Furthermore, we present a robust noise-tolerant variant of \getuserresponse for this stochastic setting.


\subsection{Noise Model}

Our noise model is parametrized by the variance matrix $\sigma \in \R_+^{\num,\num}$ which is \emph{unknown} to the algorithm. The acceptance function $\prob(\labelfunc((i,j,\cost))=1)$ is given by the CDF of a normal distribution $\mathcal{N}(\hypstar_{i,j}, \sigma_{i,j})$ truncated to $[a_{i,j}, b_{i,j}] = [\hypstar_{i,j} - \beta_{i,j}, \hypstar_{i,j} +  \beta_{i,j}]$ where $\beta_{i,j} = \min \{ \hypstar_{i,j}, \range - \hypstar_{i,j} \}$, \emph{i.e.},
\begin{align*}
  \prob(\labelfunc((i,j,\cost))=1) = \begin{cases}
     0 & 0 \leq c < a_{i,j}, \\
     1 & b_{i,j} < c \leq \range, \\
     \frac{\Phi\big(\tfrac{c-\hypstar_{i,j}}{\sigma_{i,j}}\big) - \Phi\big(\tfrac{a_{i,j}-\hypstar_{i,j}}{\sigma_{i,j}}\big)}{\Phi\big(\tfrac{b_{i,j}-\hypstar_{i,j}}{\sigma_{i,j}}\big) - \Phi\big(\tfrac{a_{i,j}-\hypstar_{i,j}}{\sigma_{i,j}}\big)} & \text{ otherwise}.
  \end{cases}
\end{align*}
Here, $\Phi(\cdot)$ denotes the CDF of the standard normal distribution. Equivalently, the acceptance function could have been defined as
\begin{align*}
\prob(\labelfunc((i,j,\cost)) = \labelpoint) =  
\begin{cases}
& 1 - \eta((i,j,\cost)) \qquad \text{ if } y = \indfunc{(\cost \geq \hypstar_{(i, j)})} \\
& \eta((i,j,\cost)) \qquad \qquad \text{ otherwise }, 
\end{cases}
\end{align*}
where $\eta((i,j,\cost))$ corresponds to the \emph{noise rate} and is given as
\begin{align*}
  \eta((i,j,\cost)) = \begin{cases}
    0 & 0 \leq c < a_{i,j}, \\
    0 & b_{i,j} < c \leq \range, \\\vspace{2mm}
    \frac{\Phi\big(\tfrac{c-\hypstar_{i,j}}{\sigma_{i,j}}\big) - \Phi\big(\tfrac{a_{i,j}-\hypstar_{i,j}}{\sigma_{i,j}}\big)}{\Phi\big(\tfrac{b_{i,j}-\hypstar_{i,j}}{\sigma_{i,j}}\big) - \Phi\big(\tfrac{a_{i,j}-\hypstar_{i,j}}{\sigma_{i,j}}\big)}  & a_{i,j} \leq \cost < \hypstar_{i,j}, \\
    1 - \frac{\Phi\big(\tfrac{c-\hypstar_{i,j}}{\sigma_{i,j}}\big) - \Phi\big(\tfrac{a_{i,j}-\hypstar_{i,j}}{\sigma_{i,j}}\big)}{\Phi\big(\tfrac{b_{i,j}-\hypstar_{i,j}}{\sigma_{i,j}}\big) - \Phi\big(\tfrac{a_{i,j}-\hypstar_{i,j}}{\sigma_{i,j}}\big)}  & \hypstar_{i,j} \leq \cost \leq b_{i,j}  . 
  \end{cases}
\end{align*}
Note that in our model, for the case $\hypstar_{i,j}=0$ the noise rate is zero and every offer is accepted; and for the case $\hypstar_{i,j}=\range$ the noise rate is also zero and all offers less than $\range$ are rejected.

The noise rate in our model is \emph{unbounded}, \emph{i.e.}, at $c=\hypstar_{i,j}$ the acceptance function $\prob(\labelfunc((i,j,\cost))=1)=0.5$.

\subsection{Robust \getuserresponse}

We now develop a robust variant of the function \getuserresponse for our noise model. The key idea is to ensure that the maximum noise $\eta_{\max}$ \emph{experienced} by the function satisfies $\eta_{\max} < 0.5$. In this way, the noise actually experienced is  \emph{bounded} and we can easily extend the function \getuserresponse from the noise-free case by issuing repeated queries and deciding on the final label via majority voting --- this approach has been extensively studied in the active learning literature~\cite{kaariainen2006active,jamieson2011active,castro2006upper}.

{\bf \getuserresponse for bounded noise.}
Our function \getuserresponse for bounded noise is shown in Algorithm~\ref{alg.adapquerybn} --- adopted from~\cite{kaariainen2006active}. For a query $\querypoint=(i,j,\cost)$ with underlying \emph{unknown} noise rate $\eta(\querypoint)$, the number of queries to the user is
\begin{align*}
 \widetilde{\bigObound} \left( \frac{\ln(\tfrac{\num^2}{\delta})}{(0.5 - \eta(\querypoint))^2} \right).
\end{align*}

\begin{algorithm}[!t]
	\caption{\getuserresponse for bounded noise}\label{alg.adapquerybn}
\begin{algorithmic}[1]
	\STATE{{\bfseries Input:} \looseness -1 $\querypoint^t =  (i^t, j^t, \cost^t)$, \learnhm parameters $(\num, \range)$, error parameters $(\epsilon, \delta)$}
   	\STATE{{\bfseries Output:} label $\labelpoint^t$}
   	\STATE{{\bfseries Initialize:} \\
    	   iteration $l = 0$; probabilities $p_1^l = p_0^l = \tfrac{1}{2}$; counts $n_1^l = n_0^l = 0$\\
    	   $\delta' = \sfrac{\delta}{\big(n^2 \log{\frac{\range}{\epsilon}}\big)}$ \mbox{\; \; \; \;} \emph{ // needed for the PAC guarantees}
    	   }
   	\WHILE{\texttt{TRUE}}
		\STATE{$l = l + 1$}
		\STATE{$\labelpoint \gets \text{get response from user for query } \querypoint^t$}
		\STATE{update $n_1^l = n_1^{l-1} + \indfunc{(y=1)}$;  $n_0^l = n_0^{l-1} + \indfunc{(y=0)}$}
		\STATE update $p_1^l = \frac{n_1^l}{l}$; $p_0^l = \frac{n_0^l}{l}$
		\STATE $\beta^l = \sqrt{\frac{1}{2 \dot l} \ln \left(\frac{\pi^2 l^2}{3 \delta'} \right)}$
    		\IF {$p_1^l - \beta^l \geq 0.5$}
    			\STATE{{\bfseries Return:} $\labelpoint^t = 1$}
		\ENDIF    			
    		\IF {$p_1^l + \beta^l < 0.5$}
    			\STATE{{\bfseries Return:} $\labelpoint^t = 0$}
		\ENDIF    
   	\ENDWHILE	
\end{algorithmic}
\end{algorithm}

{\bf \getuserresponse for unbounded noise.} The sample complexity of the repeated query approach depends on the noise-rate $\eta(x)$. In particular, the complexity is a function of $\gamma'=\tfrac{1}{(0.5 - \eta(x))^2}$. Thus, as long as the noise-rate is bounded away from $0.5$, the repeated query approach can be used. However, in our noise model, the noise-rate is $0.5$ at $\cost = \hypstar_{i,j}$. The key idea to deal with this challenge is, given a query $\querypoint = (i, j, \cost)$, to actually issue three queries in ``parallel":  $\querypoint^{-} = (i, j, \max(0, \cost - \alpha  \epsilon)), \querypoint = (i, j, \cost), \querypoint^{+} = (i, j, \min(\range, \cost + \alpha \epsilon))$, where $\alpha \in (0, 0.5)$. One of these queries has noise rate strictly less than $0.5$. The repeated query algorithm \getuserresponse for unbounded noise is shown in Algorithm~\ref{alg.adapqueryubn}.

For a query $\querypoint=(i,j,\cost)$, the number of queries to the user is
\begin{align*}
 \widetilde{\bigObound} \left( \frac{3 \ln(\tfrac{3\num^2}{\delta})}{(0.5 - \min \{ \eta(\querypoint^-), \eta(\querypoint), \eta(\querypoint^+)  \} )^2} \right).
\end{align*}
The largest noise rate $\eta_{\max}(i,j)$ experienced by \getuserresponse for pair $(i,j)$ is bounded by
\begin{align*}
  \eta_{\max}(i,j) = \min \{ \eta(i, j, \hypstar_{i,j} ), \eta(i, j, \max(0, \hypstar_{i,j} - \alpha\epsilon )), \eta(i, j, \min(\range, \hypstar_{i,j} + \alpha\epsilon)) \}.
\end{align*}
We can now define the largest noise rate $\eta_{\max}$ experienced by \getuserresponse as $\eta_{\max}=\max_{i,j \in [\num]} \eta_{\max}(i,j)$.
Using the parameters of the noise model we can bound $\eta_{\max}$ in terms of the largest variance $\sigma_{\max} = \max_{i,j \in [\num]} \sigma_{i,j}$ by 
\begin{align*}
  \eta_{\max} &\leq \frac{\Phi\big(\tfrac{-\alpha\epsilon}{\sigma_{\max}}\big) - \Phi\big(\tfrac{-\sfrac{\range}{2}}{\sigma_{\max}}\big) }{\Phi\big(\tfrac{\sfrac{\range}{2}}{\sigma_{\max}}\big) - \Phi\big(\tfrac{-\sfrac{\range}{2}}{\sigma_{\max}}\big)  } \\
    &= \frac{\int_{-\sfrac{\range}{2}}^{-\alpha\epsilon} e^{-\frac{w^2}{2\sigma_{\max}^2}} \text{d}w }{  \int_{-\sfrac{\range}{2}}^{\sfrac{\range}{2}} e^{-\frac{w^2}{2\sigma_{\max}^2}} \text{d}w } \\
    &= \frac{1}{2} - \frac{\int_{0}^{\alpha\epsilon} e^{-\frac{w^2}{2\sigma_{\max}^2}} \text{d}w }{  \int_{-\sfrac{\range}{2}}^{\sfrac{\range}{2}} e^{-\frac{w^2}{2\sigma_{\max}^2}} \text{d}w }.  
\end{align*}
The worst case number of queries to the user is characterized by the quantity 
\begin{align*}
   \gamma = \frac{3 \ln(\tfrac{3\num^2}{\delta})}{(0.5 - \eta_{\max} )^2}.
\end{align*}

\begin{algorithm}[!t]
	\caption{\getuserresponse for unbounded noise}\label{alg.adapqueryubn}
\begin{algorithmic}[1]
	\STATE{{\bfseries Input:} \looseness -1 $\querypoint^t =  (i^t, j^t, \cost^t)$, \learnhm parameters $(\num, \range)$, error parameters $(\epsilon, \delta)$}
   	\STATE{{\bfseries Output:} data point $\datapoint^t = ((i^t, j^t, \overline{\cost}^t), \labelpoint^t)$}
   	\STATE{{\bfseries Initialize:} \\
	   $\delta'' = \tfrac{\delta}{3} \tfrac{\log(\tfrac{\range}{\epsilon})}{\log( \tfrac{\range}{\epsilon(1-2\alpha)} )}$  \mbox{\; \; \; \;} \emph{ // needed for the PAC guarantees}
 \\
	   $\alpha = \tfrac{1}{3}$ \\
 	   $\querypoint^t_1 =  (i^t, j^t, \cost^t)$ \\
 	   $\querypoint^t_2 =  (i^t, j^t, \max(0, \cost^t - \alpha \epsilon))$ \\
 	   $\querypoint^t_3 =  (i^t, j^t, \min(\range, \cost^t + \alpha \epsilon))$ \\ 	   
    	   iteration $l = 0$; 
    	   }
   	\WHILE{\texttt{TRUE}}
		\STATE{$l = l + 1$}\\
		\STATE{Invoke \getuserresponse in Algorithm~\ref{alg.adapquerybn} for $\querypoint^t_1$, $\querypoint^t_2$ and $\querypoint^t_3$ with noise parameters $(\epsilon, \delta'')$ in parallel; terminate execution once any of these invocations terminated --- let $m$ be the \emph{index} of the invocation that terminated first and let $y^t$ be its returned label}
		\STATE{{\bfseries Return:} data point $\datapoint^t = (x^t_m, y^t)$}
   	\ENDWHILE	
\end{algorithmic}
\end{algorithm}







\section{Speeding up \learnhm} \label{app.algoPart4-Extensions-Speedup}

\begin{algorithm}[!t]
	\caption{Faster variant of \luproj}\label{alg.luprojx}
\begin{algorithmic}[1]
	\STATE{{\bfseries Input:} \looseness -1 $\widetilde{L}^{t}, \widetilde{U}^{t}$, $\texttt{Pivots}, \texttt{STpairs}$}
	\STATE{{\bfseries Output:} $L^{t}, U^{t}$}
   	\STATE{$U^{t} \gets  \texttt{\uproj}(\widetilde{U}^{t}, \texttt{Pivots}, \texttt{STpairs})$}
   	\STATE{$L^{t} \gets  \texttt{\lproj}(\widetilde{L}^{t}, U^{t}, \texttt{Pivots}, \texttt{STpairs})$}
   	\STATE{{\bfseries Return:} $L^{t}, U^{t}$}
\end{algorithmic}
\end{algorithm}

\begin{algorithm}[!t]
	\caption{Faster variant of \uproj}\label{alg.uprojx}
\begin{algorithmic}[1]
	\STATE{{\bfseries Input:} \looseness -1 $\widetilde{U}^{t}$, $\texttt{Pivots}, \texttt{STpairs}$}
	\STATE{{\bfseries Output:} $U^{t}$}
  	 \STATE{{\bfseries Initialize:} $U^{t} =  \widetilde{U}^{t}$}
	\FOR{$k \in \texttt{Pivots}$}
		\FOR{$(i,j) \in \texttt{STpairs}$}
			\STATE{ $U^{t}_{i, j} = \min {\left(U^{t}_{i, j}, U^{t}_{i, k} + U^{t}_{k, j}\right)}$ }
		\ENDFOR
	\ENDFOR
   	\STATE{{\bfseries Return:} $U^{t}$}
\end{algorithmic}
\end{algorithm}

\begin{algorithm}[!t]
	\caption{Faster variant of \lproj}\label{alg.lprojx}
\begin{algorithmic}[1]
	\STATE{{\bfseries Input:} \looseness -1 $\widetilde{L}^{t}, U^{t}$, $\texttt{Pivots}, \texttt{STpairs}$}
	\STATE{{\bfseries Output:} $L^{t}$}	
  	 \STATE{{\bfseries Initialize:} $L^{t} =  \widetilde{L}^{t}$}
	\FOR{$k \in \texttt{Pivots}$}
		\FOR{$(i,j) \in \texttt{STpairs}$}
				\STATE{ $L^{t}_{i, j} = \max {\left(L^{t}_{i, j}, L^{t}_{i, k} - U^{t}_{j, k}, L^{t}_{k, j} - U^{t}_{k, i}\right)}$ }		
		\ENDFOR
	\ENDFOR
   	\STATE{{\bfseries Return:} $L^{t}$}
\end{algorithmic}
\end{algorithm}

The key idea for speeding up \learnhm is that we can choose the constraints that should be exploited instead of exploiting all the constraints after every query (line 6 in \uproj \& \lproj). If the constraints to be exploited are selected cleverly, we can still get reasonable benefits from tightening lower and upper bounds. For \learnhm this can be achieved as follows:
\begin{itemize}
\item First, we do not need to invoke \luproj after every query (this is equivalent to not exploiting any violated constraint). At any iteration $t$, \learnhm invokes $\luproj$ only if $\widetilde{U}^t_{i^t,j^t} - \widetilde{L}^t_{i^t,j^t} \leq \epsilon$. Otherwise, it simply sets $(L^t, U^t) \gets (\widetilde{L}^t, \widetilde{U}^t)$.
\item Second, when \learnhm invokes \luproj at iteration $t$ we only consider the $2\num$ constraints which involve $i^t,j^t$ in lines 4 and 5 of Algorithms~\ref{alg.uproj} and~\ref{alg.lproj}.
\end{itemize}

The faster variant of \luproj is given by the function in Algorithm~\ref{alg.luprojx}, which in turn calls the faster variant of \uproj (Algorithm~\ref{alg.uprojx}) and the faster variant of \lproj (Algorithm~\ref{alg.lprojx}). The function \luproj in Algorithm~\ref{alg.luprojx} has the additional inputs $\texttt{Pivots}, \texttt{STpairs}$.  By calling \luproj with $\texttt{Pivots}=\emptyset, \texttt{STpairs}=\emptyset$, this is equivalent to not exploiting any violated constraint.  
At iteration $t$, after \learnhm invoked the query $(i^t, j^t, \cost^t)$ proposed by \qclique and updated the local bounds, \learnhm invokes \luproj (when $\widetilde{U}^t_{i^t,j^t} - \widetilde{L}^t_{i^t,j^t} \leq \epsilon$) with $\texttt{STpairs} = (\{ \max(i^t,j^t )\} \times \clique) \cup ( \clique \times \{ \max(i^t,j^t \} ) $ and sets $\texttt{Pivots} = \{ \min(i^t,j^t) \}$.

\learnhm implementing this idea invokes $\luproj$ at most $\num^2$ times with a runtime of $4 \num$. Hence, the total runtime of \learnhm is $\Thetabound(\num^3)$. Most importantly, the sample complexity bounds from the previous section still apply. In fact, the proofs for the sample complexity bounds analyze the speeded up variant of \learnhm because of ease of analysis.

Furthermore, we can make the following observations:
\begin{itemize}
\item The function \luproj in Algorithm~\ref{alg.luprojx} is a strict generalization of  Algorithm~\ref{alg.luproj} --- by passing $\texttt{Pivots}=[\num], \texttt{STpairs}=[\num]^2$, it reduces to Algorithm~\ref{alg.luproj}. 
\item  The resulting output $L^t$ and $U^t$ may not be the optimal solutions to the Problems~\ref{sec.algoPart2-LUBounds.proj.6} and~\ref{sec.algoPart2-LUBounds.proj.5}, respectively.
\end{itemize} 

\section{Experimental Results for Stochastic Settings} \label{app.experiments}

\begin{figure*}[!t]
  \begin{subfigure}[b]{0.33\textwidth}
    \centering
    \includegraphics[width=0.95\textwidth]{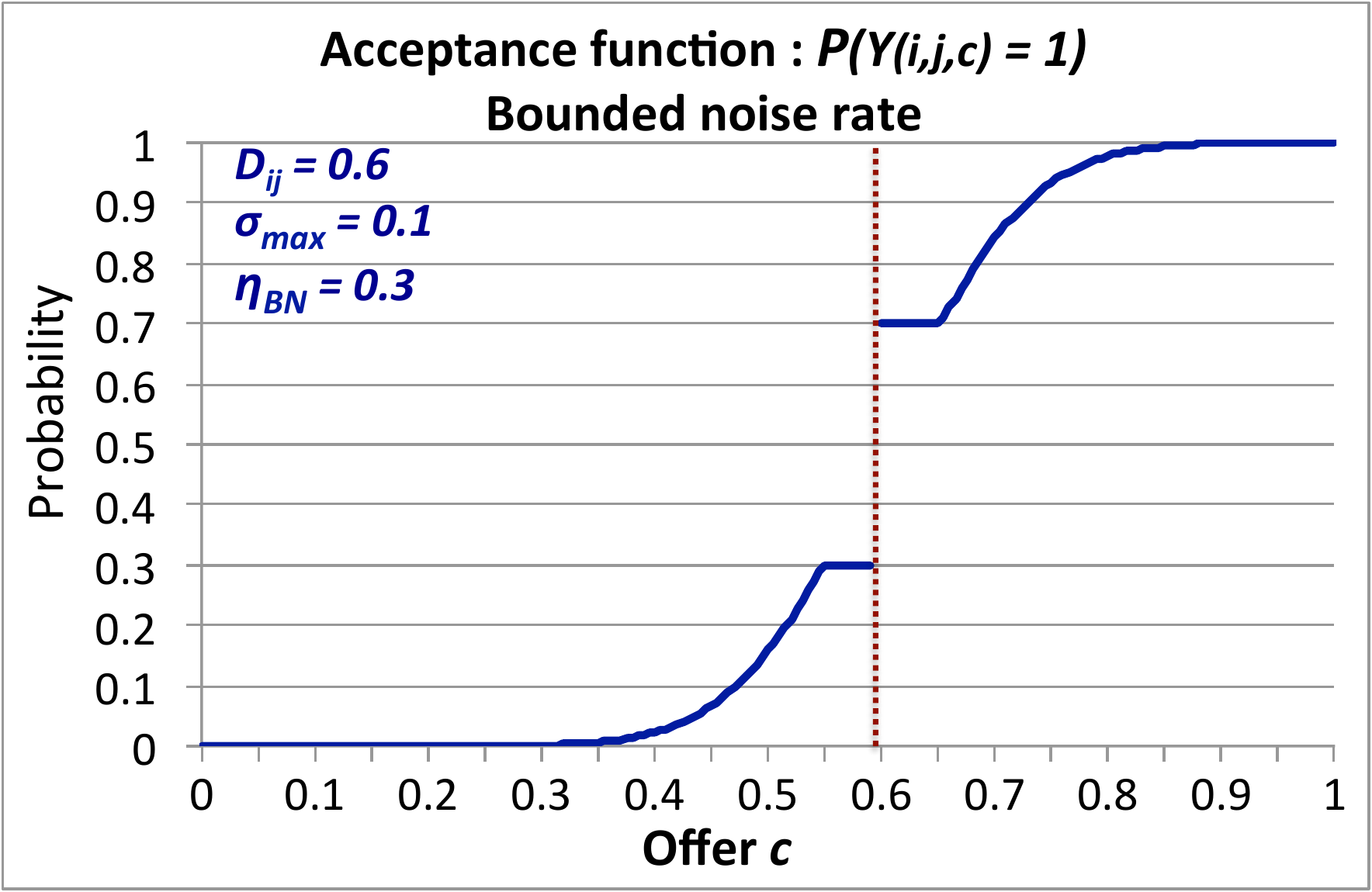}
    \caption{Acceptance function for bounded noise}
    \label{fig.cdf_bounded-noise}
  \end{subfigure}
  \begin{subfigure}[b]{0.33\textwidth}
    \centering
    \includegraphics[width=0.95\textwidth]{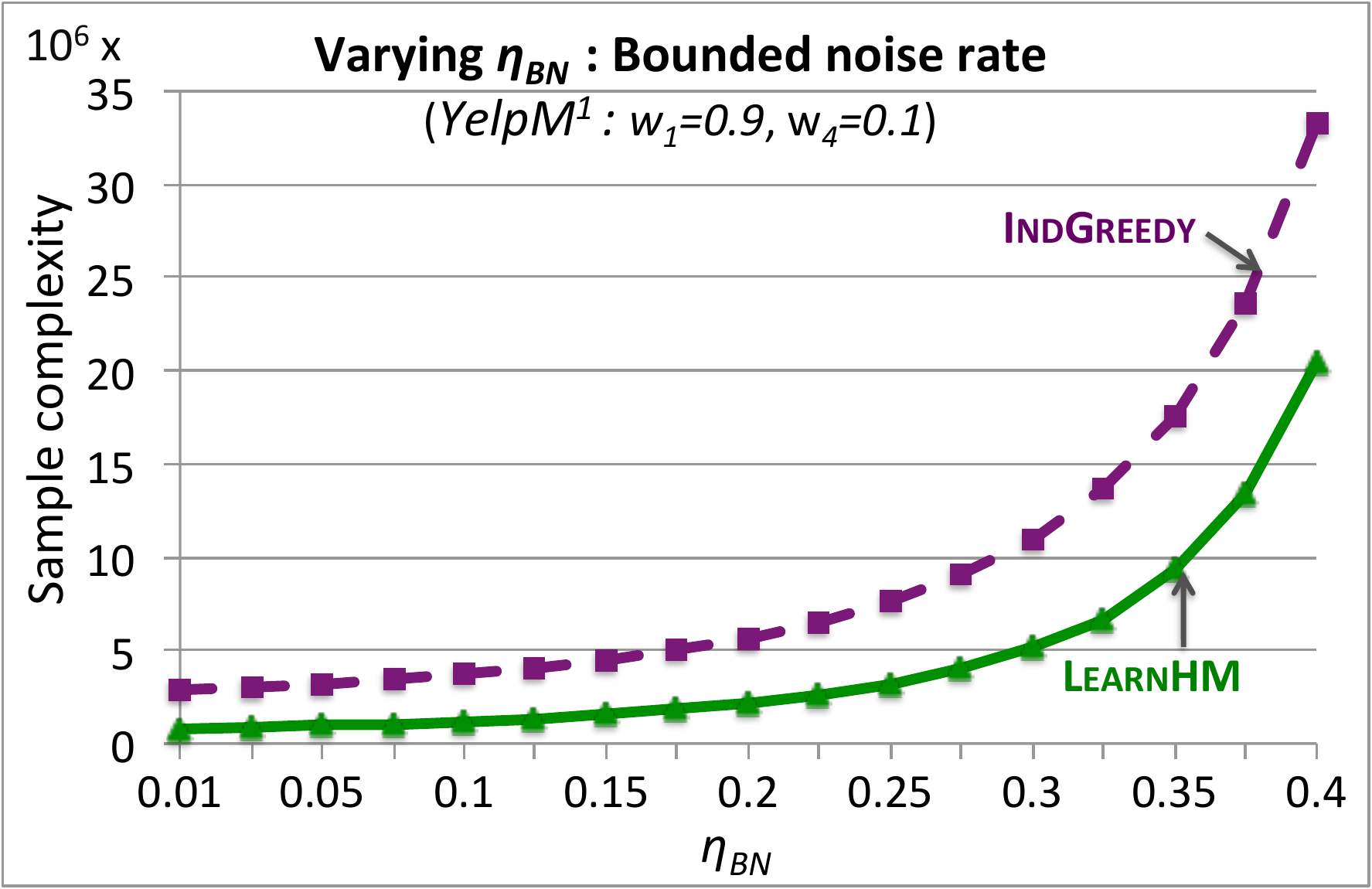}
    \caption{Varying bounded noise $\eta_{\text{BN}}$ (\emph{YelpM}$^1$)}
    \label{fig.yelp_model1_bnoise_vary-bnoisecutoff}
  \end{subfigure}
  \begin{subfigure}[b]{0.33\textwidth}
    \centering
    \includegraphics[width=0.95\textwidth]{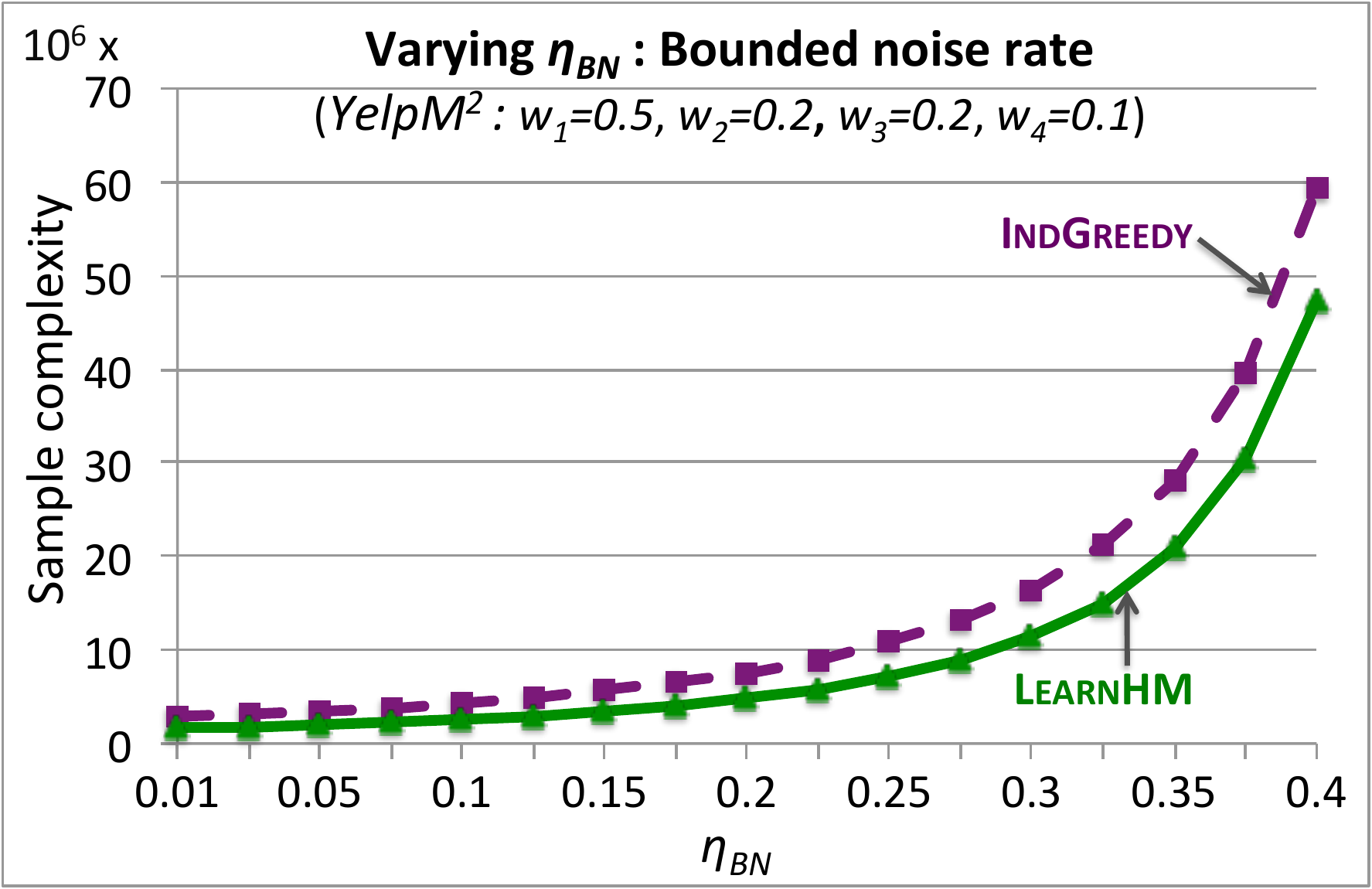}
    \caption{Varying bounded noise $\eta_{\text{BN}}$ (\emph{YelpM}$^2$)}
    \label{fig.yelp_model2_bnoise_vary-bnoisecutoff}
  \end{subfigure}
  \begin{subfigure}[b]{0.33\textwidth}
    \centering
    \includegraphics[width=0.95\textwidth]{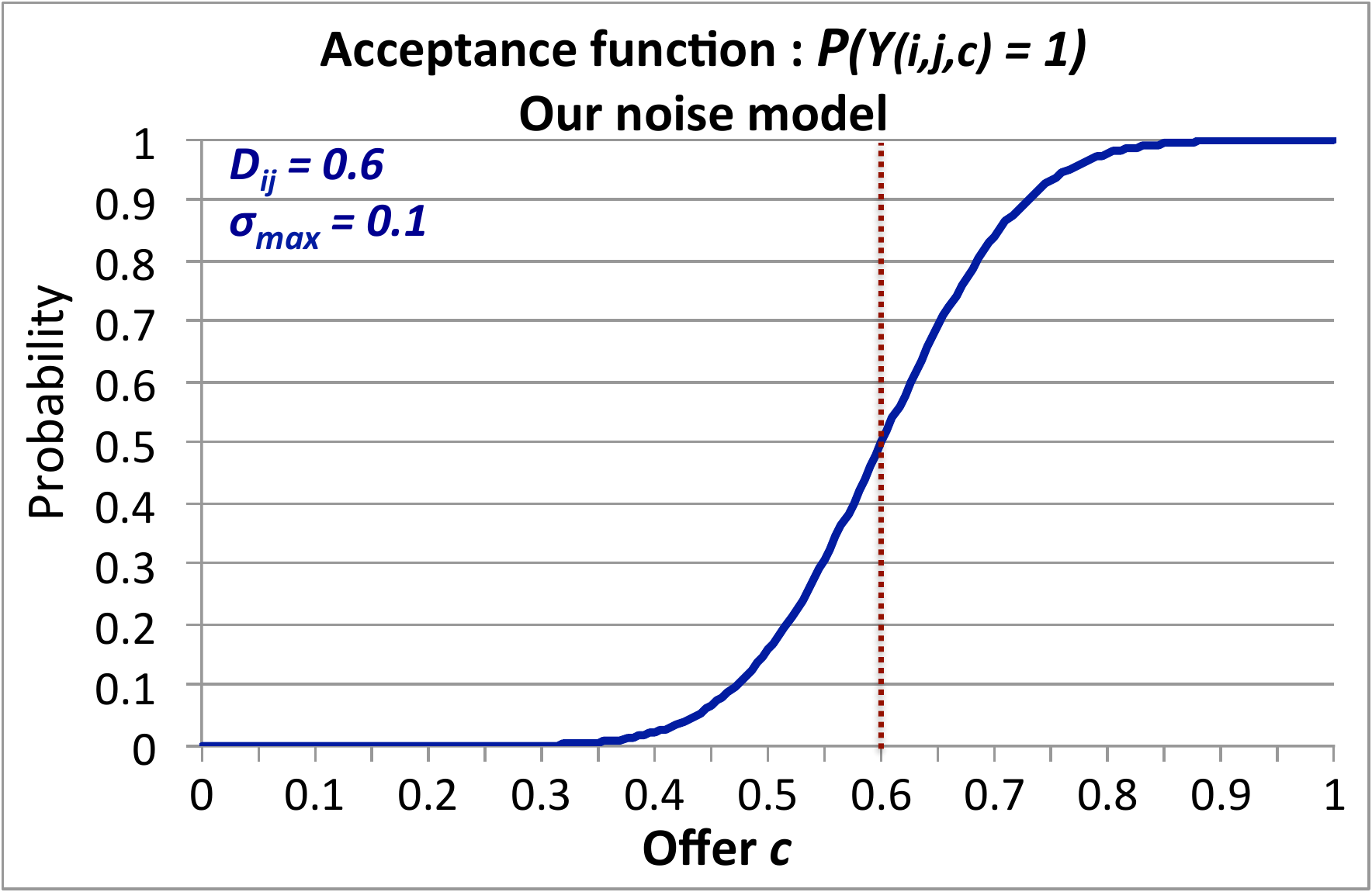}
    \caption{Acceptance function for our noise model}
    \label{fig.cdf_unbounded-noise}
  \end{subfigure}
  \begin{subfigure}[b]{0.33\textwidth}
    \centering
    \includegraphics[width=0.95\textwidth]{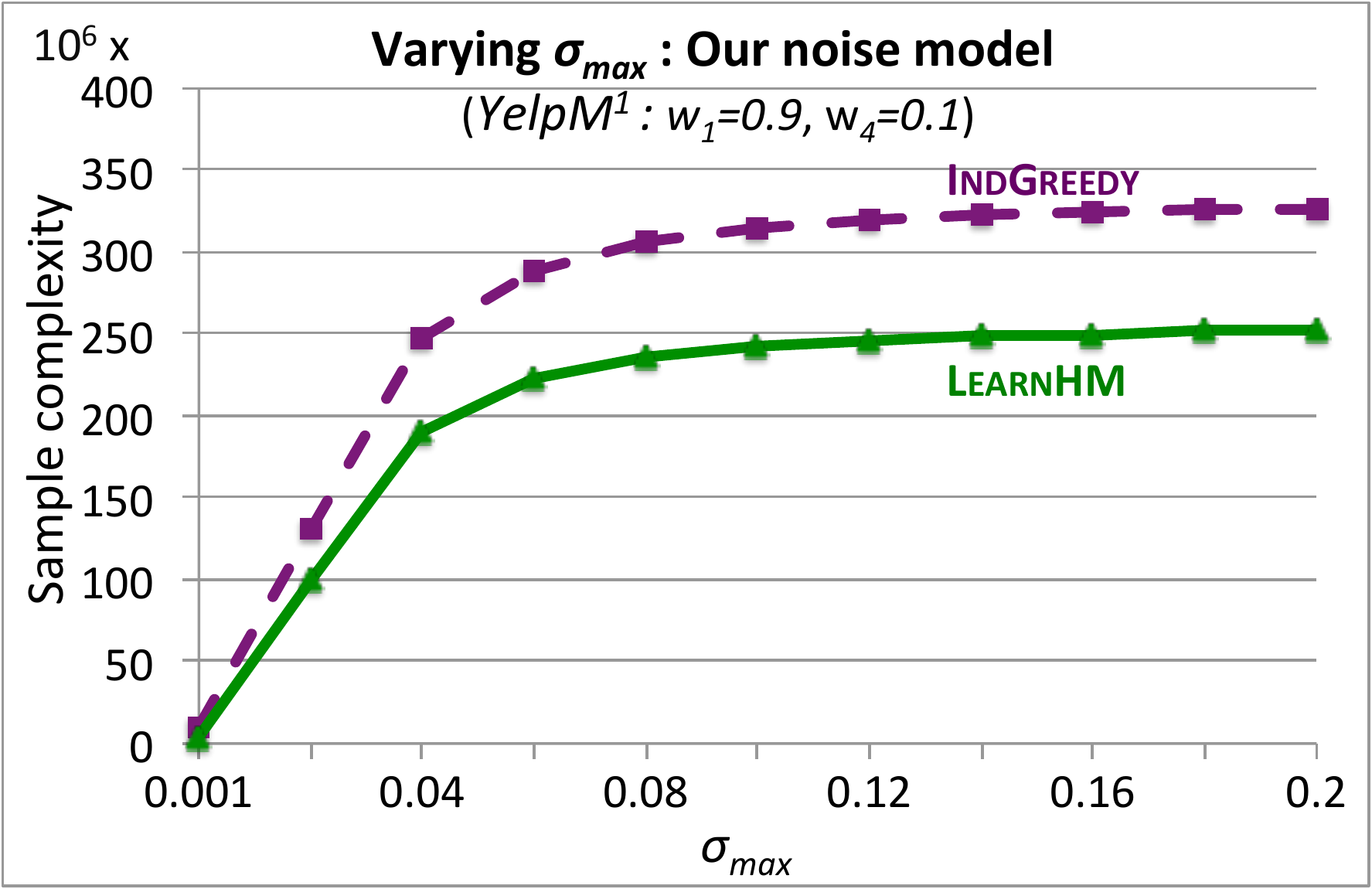}
    \caption{Varying noise variance $\sigma$ (\emph{YelpM}$^1$)}
    \label{fig.yelp_model1_unbnoise_vary-noisesigma}
  \end{subfigure}
  \begin{subfigure}[b]{0.33\textwidth}
    \centering
    \includegraphics[width=0.95\textwidth]{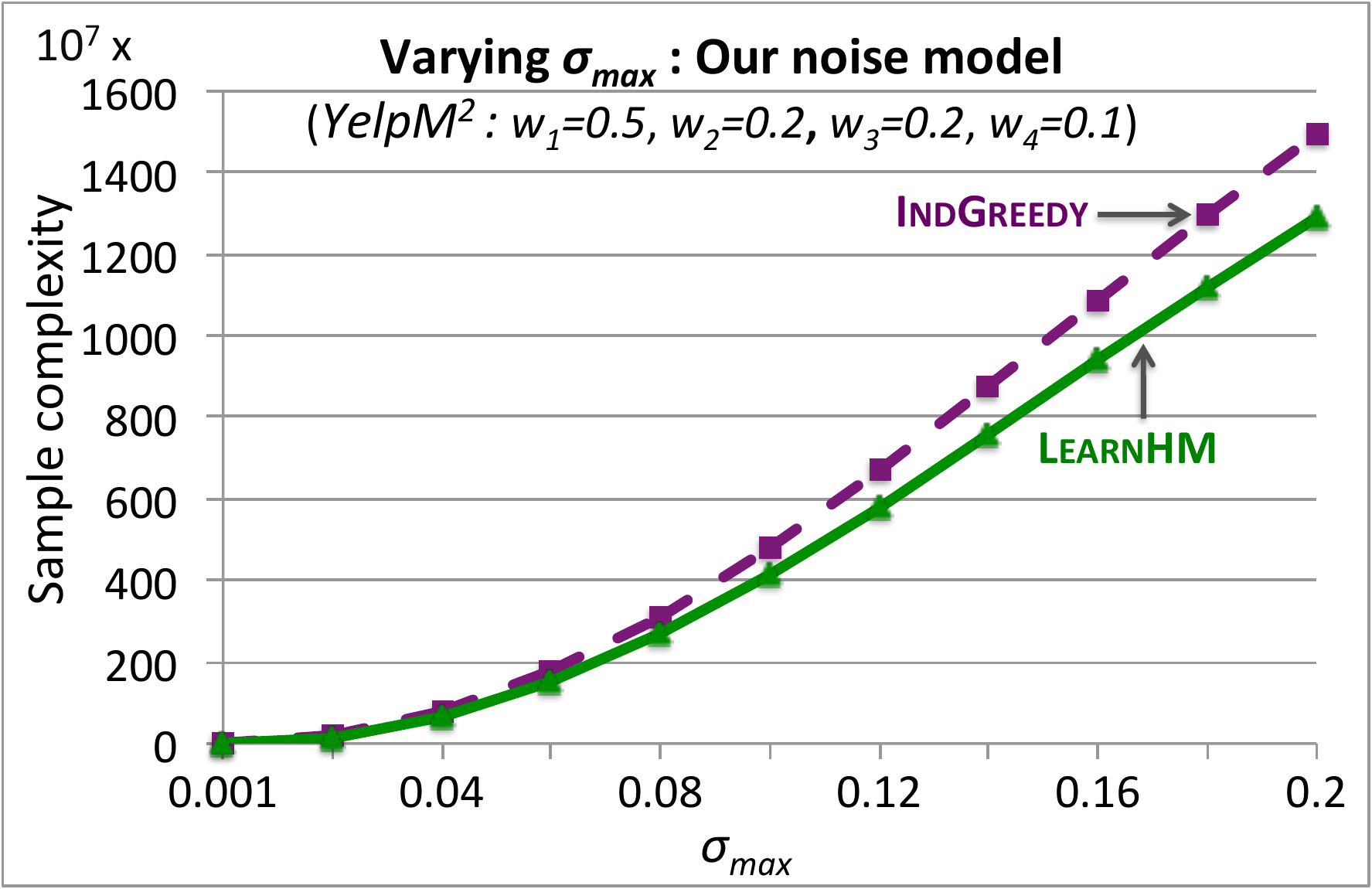}
    \caption{Varying noise variance $\sigma$ (\emph{YelpM}$^2$)}
    \label{fig.yelp_model2_unbnoise_vary-noisesigma}
  \end{subfigure}
  \vspace{-7mm}
  \caption{Sample complexity results in the stochastic settings for different user preference models defined via the \emph{Yelp} dataset. (a)--(c) results are for the stochastic setting with bounded noise rate parameterized  by $(\eta_{max}, \sigma)$ and (d)--(f)  results are for the stochastic setting with our noise model parameterized by $\sigma$.}
  \label{fig.results.noise}
\end{figure*}

We perform two sets of experiments for the stochastic setting. 
In the stochastic setting with bounded noise rate (\emph{i.e.}, the algorithm knows that the noise rate is bounded away from $0.5$), \learnhm uses the function \getuserresponse for bounded noise in Algorithm~\ref{alg.adapquerybn}. In our noise model (\emph{i.e.}, the algorithm makes no assumption about the bound on the noise rate), \learnhm uses the function \getuserresponse for unbounded noise in Algorithm~\ref{alg.adapqueryubn}.\\
Figures~\ref{fig.results.noise} (a-c) show results for the stochastic setting with bounded noise rate parameterized by $(\eta_{\text{BN}}, \sigma)$. Then, we present results for our noise model in Figures~\ref{fig.results.noise} (d-f) parameterized by $\sigma$. We used the same value of variance for all the pairs, \emph{i.e.}, $\forall i,j \in [\num]\colon \sigma_{i,j} = \sigma_{max}$. The PAC-parameters are set to $(\epsilon=0.01, \delta=0.05)$. The other parameters are set similar to the noise-free setting, \emph{i.e.},  $\num = 100$, $\range=1$.

{\bf Stochastic acceptance function.}
Figure~\ref{fig.cdf_unbounded-noise} shows the acceptance function for $\sigma_{max} = 0.1$ when $\hypstar_{i,j} = 0.6$. Figure~\ref{fig.cdf_bounded-noise} shows the corresponding acceptance function for the bounded noise setting with the noise rate $\eta_{\text{BN}} = 0.3$. \\
{\bf Results for bounded noise rate.}
Figure~\ref{fig.yelp_model1_bnoise_vary-bnoisecutoff} and Figure~\ref{fig.yelp_model2_bnoise_vary-bnoisecutoff} show the results for bounded noise rate on the two preference models \emph{YelpM}$^1$ and \emph{YelpM}$^2$ respectively. We vary the bound on the noise rate $\eta_{\text{BN}}$, and the value of $\sigma_{max}$ is fixed to $0.1$.
We observe that the sample complexity increases drastically for all the algorithms as $\eta_{\text{BN}}$ approaches $0.5$. \learnhm outperforms \naive for both user preference models, though the relative gain of sample complexity improvements of \learnhm in comparison to \naive decreases with increasing $\eta_{\text{BN}}$. Note that as $\eta_{\text{BN}}$ increases, the cost of repeated querying near the threshold boundary dominates the sample complexity and tends to infinity.\\
{\bf Results for our noise model.}
Figure~\ref{fig.yelp_model1_unbnoise_vary-noisesigma} and Figure~\ref{fig.yelp_model2_unbnoise_vary-noisesigma} show the results for our noise model by varying the variance $\sigma_{max}$, on  \emph{YelpM}$^1$ and \emph{YelpM}$^2$ respectively. 
\getuserresponse for unbounded noise in Algorithm~\ref{alg.adapqueryubn} ensures that the sample complexity stays bounded away from $\infty$. As we increase $\sigma_{max}$, the CDF eventually resembles the uniform noise setting, and the sample complexity curve saturates. 

Most importantly, \learnhm constantly outperforms \naive across different noise settings and different user preference models.




\section{Sample Complexity Analysis} 
\label{app.algoPart3-Query}
\begin{proof}[{\bf Proof of Theorem~\ref{theorem.complexity1}}]

Consider the current clique $\clique$, and a new item $a \in \itemset \setminus \clique$ that should be added to $\clique$. In order to do so, the policy needs to learn all distances between $a$ and all items in $\clique$. During this process of learning all distances between $a$ and all items in $\clique$, let us say that \learnhm{} / \qclique{} has already learnt all distances between $a$ and items $\clique' \subset \clique$, and is now learning the distance between $a$ and $b \in \clique \setminus \clique'$. That is, \learnhm{} / \qclique{} is growing $\clique'$ until $\clique' = \clique$, and then item $a$ will be added. We will consider the following cases: 
\begin{enumerate}[label={\bf C{\arabic{enumi}}}]
  \item $b$ is the first item of its cluster that will be added to $\clique'$.
  \item Another item $e$ from the same cluster as $b$ is already part of $\clique'$.
	\begin{enumerate}[label={\bf C{\arabic{enumi}.\arabic*}}]
  	  \item $b$ belongs to the same cluster as $a$
	  \item $b$ belongs to a different cluster as $a$
	\end{enumerate}
\end{enumerate}

{\bf C1: } In order to learn the distances between $(a, b)$ and $(b,a)$, the number of queries is bounded by 
\begin{align*}
  \Big\lceil \log\Big(\frac{\range}{\epsilon}\Big) \Big\rceil.
\end{align*}

{\bf C2.1: } In order to learn the distance between $(a,b)$ or $(b,a)$, the number of queries is bounded by 
\begin{align*}
  \Big\lceil \log\Big(\frac{2\rangeintra + 2\epsilon}{\epsilon}\Big) \Big\rceil.
\end{align*}
To see this, note that
\begin{align*}
  U_{a,b} &\leq U_{a,e} + U_{e,b} \\
          &\leq (\rangeintra + \epsilon) + (\rangeintra + \epsilon) \\
          &= 2 \rangeintra + 2 \epsilon.
\end{align*}
Since $L_{a,b} \geq 0$, and our goal is to shrink the gap $U_{a,b} - L_{a,b}$ to at most $\epsilon$, we need at most the above  numbers of queries. The same argument clearly also holds for shrinking $U_{b,a} - L_{b,a}$ to at most $\epsilon$.

{\bf C2.2: } In order to learn distance between ($a,b)$ or $(b,a)$, the number of queries is bounded by 
\begin{align*}
  \Big\lceil \log\Big(\frac{2\rangeintra + 3\epsilon}{\epsilon}\Big) \Big\rceil.
\end{align*}
To see this, note that we can subtract the inequalities $U_{a,b} \leq U_{a,e} + U_{e,b}$ and $L_{a,b} \geq L_{a,e} - U_{b,e}$ to get
$U_{a,b} - L_{a,b} \leq (U_{a,e} - L_{a,e}) + U_{e,b} + U_{b,e}$. We bound the resulting inequality as follows:
\begin{align*}
  U_{a,b} - L_{a,b} &\leq (U_{a,e} - L_{a,e}) + U_{e,b} + U_{b,e} \\
          &\leq (\epsilon) + (\rangeintra + \epsilon) + (\rangeintra + \epsilon) \\
          &= 2 \rangeintra + 3 \epsilon.
\end{align*}
Since $L_{a,b} \geq 0$, and our goal is to shrink the gap $U_{a,b} - L_{a,b}$ to at most $\epsilon$, we need at most the above  numbers of queries. A similar argument holds for shrinking $U_{b,a} - L_{b,a}$ to at most $\epsilon$.

Case {\bf C1}  can only occur $\min \{ |\clique| - 1, \clusters \}$ times. Case {\bf C2.1}, and {\bf C2.2} can occur up to $|\clique|$ times. Hence, the total cost of adding item $a$ to $\clique$ is bounded by
\begin{align*}
 2 \min \{ |\clique| - 1, \clusters \} \Big\lceil \log\Big(\frac{\range}{\epsilon}\Big) \Big\rceil + 2 |\clique| \Big\lceil \log\Big(\frac{\max \{ 2\rangeintra + 2\epsilon, 2\rangeintra + 3\epsilon \}}{\epsilon}\Big) \Big\rceil.
\end{align*}
Given that we will add $\num$ items to $\clique$, the total sample complexity of \learnhm is 
\begin{align*}
  &\sum_{|\clique|=1}^{\num-1} 2 \min \{ |\clique| - 1, \clusters \} \log\Big(\frac{\range}{\epsilon}\Big) + 2 |\clique| \log\Big(\frac{\max \{ 2\rangeintra + 2\epsilon, 2\rangeintra + 3\epsilon \}}{\epsilon}\Big) \\
  &\leq \sum_{|\clique|=1}^{\num-1} 2 \min \{ |\clique| - 1, \clusters \} \Big\lceil \log\Big(\frac{\range}{\epsilon}\Big) \Big\rceil + 2 |\clique| \Big\lceil \log\Big(\frac{ 2\rangeintra + 3\epsilon }{\epsilon}\Big) \Big\rceil.
\end{align*}
This gives us the bound on the sample complexity from Theorem~\ref{theorem.complexity1}, \emph{i.e.},
\begin{align*}
  2 \num  \clusters \Big\lceil \log\Big(\frac{\range}{\epsilon}\Big) \Big\rceil + \num^2 \Big\lceil \log\Big(\frac{2\rangeintra + 3\epsilon}{\epsilon}\Big) \Big\rceil.
\end{align*}

\end{proof}

\begin{proof}[{\bf Proof of Corollary~\ref{theorem.complexity2}}]

This proof follows the same lines as that of Theorem~\ref{theorem.complexity1}. Therefore, we only show the differences in the proof. Since we invoked the algorithm with $\epsilon < \Delta$, we can for any $i,j \in [\num]$ collapse the gap $U_{i,j} - L_{i,j}$ to zero whenever $U_{i,j} - L_{i,j} < \Delta$.

{\bf C1: } In order to learn the distances between $(a, b)$ and $(b,a)$, the number of queries is bounded by 
\begin{align*}
  \Big\lceil \log\Big(\frac{\range}{\Delta}\Big) \Big\rceil.
\end{align*}

{\bf C2.1: } In this setting, the number of queries is actually zero. To see this, note  that
\begin{align*}
  U_{a,b} &\leq U_{a,e} + U_{e,b} \\
          &\leq 0 + 0 \\
          &= 0.
\end{align*}

{\bf C2.2: } In this setting, the number of queries is again zero. To see this, note  that
\begin{align*}
  U_{a,b} - L_{a,b} &\leq (U_{a,e} - L_{a,e}) + U_{e,b} + U_{b,e} \\
          &\leq 0 + 0 + 0 \\
          &= 0.
\end{align*}

This leads to the total sample complexity 
\begin{align*}
  2 \num  \clusters \Big\lceil \log\Big(\frac{\range}{\Delta}\Big) \Big\rceil.
\end{align*}

To this end, we prove a lower bound on the sample complexity. Therefore, consider a more powerful algorithm which knows that there are $\clusters$ clusters, as well as one item from every cluster. We assume that the clusters are balanced equally. Now, given an item $i$, learning all the distances to all the other items, is equivalent to assigning the item $i$ to its cluster. On average $\tfrac{\clusters}{2}$ queries are need to perform this for item $i$, giving a lower bound of $\Omega(\num \clusters)$ --- even for this powerful variant of the algorithm with additional information. 
\end{proof}


%
%
%

\begin{proof}[{\bf Proof of Theorem~\ref{theorem.complexity3}}]
Note that we introduced $\gamma=\tfrac{3 \ln(\tfrac{3\num^2}{\delta})}{(0.5 - \eta_{\max} )^2}$ in Appendix~\ref{app.algoPart4-Extensions-Noise} which quantifies the maximum number of queries to the user in the unbounded noise case when using the function \getuserresponse in Algorithm~\ref{alg.adapqueryubn}.

The only key insight that is missing to prove this theorem is to bound the number of invocations of \getuserresponse with parameter $\alpha$ needed to shrink a gap $g^0:=U_{i,j} - L_{i,j}$ to be at most $\epsilon$. For this, note that one invocation to \getuserresponse with $\cost^t = \tfrac{1}{2} (L_{i,j} + U_{i,j})$ may not return the label for $\cost$ but for $\cost \pm \alpha \epsilon$. Thus, after updating the bounds, the gap $g^1$ is in the worst-case $\tfrac{g^0}{2} + \alpha\epsilon$. Thus, to ensure that the gap is shrunk below $\epsilon$, at most $\lceil \log \tfrac{g^0}{\epsilon(1-2\alpha)} \rceil$ invocations of \getuserresponse are necessary --- in contrast, using binary search in the noise-free case, at most $\lceil \log \tfrac{g^0}{\epsilon} \rceil$ invocations are neeeded.

To finish the proof, consider the sample complexity from Theorem~\ref{theorem.complexity1} from the noise-free setting and set $\alpha = \tfrac{1}{3}$. We have to adopt it as follows:
\begin{itemize}
  \item All terms quantifying the numbers of invocations of \getuserresponse to shrink a gap $g$ to at most $\epsilon$, \emph{i.e.}, $\lceil \log(\tfrac{g}{\epsilon}) \rceil$, are replaced by $\lceil \log \tfrac{3g}{\epsilon} \rceil$.
  \item Each invocation has a cost of $\widetilde{\bigObound} (\gamma)$.
\end{itemize}
Thus, with probability $1-\delta$, \learnhm learns $\hypstar$ to precision $\epsilon$ and has sample complexity 
\begin{align*}
   \widetilde{\bigObound} \Big( \gamma \Big( 2 \num  \clusters \Big\lceil \log\Big(\frac{3 \range}{\epsilon}\Big) \Big\rceil + \num^2 \Big\lceil \log\Big(\frac{6\rangeintra + 9\epsilon}{\epsilon}\Big) \Big\rceil \Big) \Big).
\end{align*}
\end{proof}

\begin{proof}[{\bf Proof of Corollary~\ref{theorem.complexity4}}]

Apply the same arguments from Theorem~\ref{theorem.complexity3} to Corollary~\ref{theorem.complexity2}.
\end{proof}

%

\section{Set \lclass \ for Lower Bounds and \uclass \ for Upper Bounds} \label{app.classLU}

We now provide a formal reasoning for using the sets $\lclass$ and $\uclass$ as the lower bounds $L$ and upper bounds $U$ for projection in Problems~\ref{sec.algoPart2-LUBounds.proj.6} and~\ref{sec.algoPart2-LUBounds.proj.5}, respectively.


In order to formally get the set $\uclass$ \ for upper bounds, for any $i, j, k \in [\num]$, we can state:
\begin{align}
&\hypstar_{i, j} \leq \hypstar_{i, k} + \hypstar_{k, j}  \leq  U_{i, k} + U_{k, j}  \notag \\
\implies &\hypstar_{i, j} \leq \min_{k' \in [\num]} \left(U_{i, k'} + U_{k', j} \right) \notag \\
\implies &U_{i, j} = \min_{k' \in [\num]} \left(U_{i, k'} + U_{k', j} \right) \notag \\
\implies & U_{i, j} \leq U_{i, k} + U_{k, j} \quad \forall k \in [\num] \label{app.classLU.classU.1}
\end{align}

Using Equation~\ref{app.classLU.classU.1}, we define the set of matrices for the upper bounds as follows. Note that this corresponds to the set of hemimetrics  $\hypclass$ (Equation~\ref{eq.triangle.inequality}). Remember, 
\begin{align*}
 \hypclass := \{U \in \R^{\num \times \num} \mid  U_{i, i} = 0, \; 0 \leq U_{i,j} \leq \range, U_{i, j} \leq U_{i, k} + U_{j, k} \quad \forall \ i, j, k \in [\num] \}.
\end{align*}

In order to formally get the set $\lclass$  for lower bounds, for any $i, j, k \in [\num]$, we can state:
\begin{align}
&\hypstar_{i, k} \leq \hypstar_{i, j} + \hypstar_{j, k} \notag	 \\
\implies &\hypstar_{i, j} \geq \hypstar_{i, k}  -  \hypstar_{j, k} \geq L_{i, k}  -  U_{j, k}  \notag	 \\
\implies &\hypstar_{i, j} \geq \max_{k' \in [\num]} \left(L_{i, k'} - U_{j, k'} \right) \label{app.classLU.classL.1}
\end{align}
Similarly, we can state:
\begin{align}
&\hypstar_{k, j} \leq \hypstar_{k, i} + \hypstar_{i, j}  \notag	 \\
\implies &\hypstar_{i, j} \geq \hypstar_{k, j}  -  \hypstar_{k, i} \geq L_{k, j}  -  U_{k, i}  \notag	 \\
\implies &\hypstar_{i, j} \geq \max_{k' \in [\num]} \left(L_{k', j} - U_{k', i} \right) \label{app.classLU.classL.2}
\end{align}
Combining Equations~\ref{app.classLU.classL.1}, \ref{app.classLU.classL.2} from above, we get:
\begin{align}
&\hypstar_{i, j} \geq \max_{k' \in [\num]} \big( \max \left( L_{i, k'} - U_{j, k'}, L_{k', j} - U_{k', i} \right) \big) \notag \\
\implies 
&L_{i, j} = \max_{k' \in [\num]} \big( \max \left( L_{i, k'} - U_{j, k'}, L_{k', j} - U_{k', i} \right) \big) \notag \\
\implies
& L_{i, j} \geq \max \left( L_{i, k} - U_{j, k}, L_{k, j} - U_{k, i} \right) \quad \forall k \in [\num]  \label{app.classLU.classL.3}
\end{align}
Using Equation~\ref{app.classLU.classL.3}, we define the set $\lclass$  parameterized by the upper bound matrix $U \in \uclass$ as follows:
\begin{align*}
 \lclass(U) := \{L \in \R^{\num \times \num} \mid &L_{i, i} = 0, \; 0 \leq L_{i,j} \leq U_{i,j}, L_{i, j} \geq \max \left( L_{i, k} - U_{j, k}, L_{k, j} - U_{k, i} \quad \forall \ i, j, k \in [\num] \right) \}
\end{align*}

While the set of upper bounds corresponds to the set $\hypclass$ of bounded hemimetrics, the set $\lclass(U)$ represents more complex dependencies. It turns out that the set $\lclass(U)$ cannot be constrained to contain only hemimetrics. We provide a counter-example below.

{\bf Counter-example.}
Here, we provide an example to demonstrate that if we restrict the set $\lclass(U)$ to hemimetrics only, we may obtain invalid lower bounds.

Consider the following underlying hemimetric $\hypstar$:
\begin{center}
$\hypstar = 
\begin{pmatrix}
  0    & 1 	  & 0.5 \\
  1    & 0 	  & 0.5 \\
  0.5  & 0.5  & 0 
\end{pmatrix}
$
\end{center}
 
%
At iteration $t=0$, we have $\dataset^0 = \emptyset$. The lower bounds $L^0$ and upper bounds $U^0$ are defined as follows:

\begin{center}
$L^0 = 
\begin{pmatrix}
  0 & 0 & 0 \\
  0 & 0 & 0 \\
  0 & 0 & 0
\end{pmatrix}
$
\end{center}

\begin{center}
$U^0 = 
\begin{pmatrix}
  0 & 1 & 1 \\
  1 & 0 & 1 \\
  1 & 1 & 0
\end{pmatrix}
$
\end{center}

At iteration $t=1$, let us consider that the algorithm picks the query $\querypoint^1 = (1, 2, 0.5)$ and receives the label $\labelpoint^1 = 0$. Let us see how to update the bounds to incorporate this labeled datapoint $\datapoint^1 = (\querypoint^1, \labelpoint^1)$. 

We have $U^1 = \widetilde{U}^1  = U^0$. We initialize $\widetilde{L}^1 = L^0$, and update  $\widetilde{L}^1_{1,2} = 0.5$. We get:


\begin{center}
$\widetilde{L}^1 = 
\begin{pmatrix}
  0 & 0.5 & 0 \\
  0 & 0   & 0  \\
  0 & 0   & 0
\end{pmatrix}
$
\end{center}

Now if we restrict the class $\lclass$ to be of set of hemimetrics $\hypclass$, then the projection of $\widetilde{L}^1$ to set $\hypclass$ will give us the following ${L}^1$:

\begin{center}
$L^1 = 
\begin{pmatrix}
  0 & 0.5 & a \\
  0 & 0   & 0  \\
  0 & 0.5 - a   & 0
\end{pmatrix}
$
\end{center}

where $a \in [0, 0.5]$.

First of all, it is clear that the updated lower bound $L^1$ is not unique as any $a \in [0, 0.5]$ represents a corresponding solution $L^1$.

Secondly, the bounds $L^1, U^1$ after the update are not valid anymore. Consider the following hemimetric $\hyp'$, where $b \in [0,1]$, given below:
\begin{center}
$\hyp' = 
\begin{pmatrix}
  0 & 1 & b \\
  1 & 0 & 1  \\
  1 & 1-b & 0
\end{pmatrix}
$
\end{center}

$\hyp'$ is in the version space for all $b$, \emph{i.e.}, $\hyp' \in \hypclass^1$. However:
\begin{itemize}
   \item If $a > 0$, then all $\hyp'$ with $b\in[0,a)$ are excluded, \emph{i.e.}, $\hyp' \notin \hypclass(L^1, U^1)$.
   \item Otherwise, if $a=0$, then all $\hyp'$ with $b\in (0.5,1]$ are excluded, \emph{i.e.}, $\hyp' \notin \hypclass(L^1, U^1)$.
\end{itemize}
 
%
%
%
%
%



\section{Proof of Theorem~\ref{thm.projection}} \label{app.thm.projection}


In this section we provide the proof of Theorem~\ref{thm.projection}. The proof builds upon three lemmas, which are presented first.

\begin{lemma} \label{lemma.projection.valid1}
Assume $L^{t-1}$ and $U^{t-1}$ satisfy $\hypclass({L}^{t-1}, {U}^{t-1}) =  \closure( \hypclass^{t-1} )$, where $\closure(\cdot)$ denotes the closure. Then, $\widetilde{L}^{t}, \widetilde{U}^{t}$ obtained by the \emph{local updates} in Algorithm~\ref{alg.learnhm} (lines 9--12) after querying $x^t=(i^t,j^t,\cost^t)$ and observing response $y^t$, satisfy $\hypclass(\widetilde{L}^{t}, \widetilde{U}^{t}) =  \closure(\hypclass^t)$.
\end{lemma}
\begin{proof}
Note that $\hypclass^t$ can be written as $\hypclass^t = \hypclass \cap \Pi^t$, where $\Pi^t \subseteq \R^{\num \times \num}$ is the intersection of half-spaces defined by the constraints of the labeled datapoints. Formally,
\begin{align*}
  \Pi^t &= \{ \hyp \in \R^{\num \times \num} \mid \indfunc(c^{t'} \geq \hyp_{i^{t'}, j^{t'}}) = y^{t'} \quad \forall t' \leq t\} \\
    &= \bigcap_{t' \leq t} \{ \hyp \in \R^{\num \times \num} \mid \indfunc(c^{t'} \geq \hyp_{i^{t'}, j^{t'}}) = y^{t'} \} \\
    &= \bigcap_{t' \leq t} \{ \hyp \in \R^{\num \times \num} \mid c^{t'} \geq \hyp_{i^{t'}, j^{t'}} \textnormal{ if }  y^{t'}  = 1, \; \textnormal{ or } c^{t'} < \hyp_{i^{t'}, j^{t'}} \textnormal{ if }  y^{t'}  = 0\}.
\end{align*}
Consequently, it follows from an easy geometric argument that $\closure(\hypclass^t) = \hypclass \cap \overline{\Pi}^t$, where
\begin{align*}
  \overline{\Pi}^t &= \bigcap_{t' \leq t} \{ \hyp \in \R^{\num \times \num} \mid c^{t'} \geq \hyp_{i^{t'}, j^{t'}} \textnormal{ if }  y^{t'}  = 1, \; \textnormal{ or } c^{t'} \leq \hyp_{i^{t'}, j^{t'}} \textnormal{ if }  y^{t'}  = 0\}.
\end{align*}

It is now easy to see the claimed equivalence of $\hypclass(\widetilde{L}^{t}, \widetilde{U}^{t})$ and $\closure(\hypclass^t)$. Consider the following two cases:
\begin{itemize}
  \item {\bf C1:} The response is given as $y^t=1$. In this case $\widetilde{U}^{t}$ is updated to $\widetilde{U}^{t} = c^t$ in Algorithm~\ref{alg.learnhm} (lines 9--12). Hence, $\hypclass(\widetilde{L}^{t}, \widetilde{U}^{t})$ consists of all hypothesis $\hyp \in \hypclass({L}^{t-1}, {U}^{t-1}) =  \closure( \hypclass^{t-1} )$ that satisfy $c^t \geq \hyp_{i^t,j^t}$. Also, $\closure(\hypclass^t) = \closure(\hyp^{t-1}) \cap \{ \hyp \in \R^{\num \times \num} \mid c^{t'} \geq \hyp_{i^{t'}, j^{t'}} \}$, \emph{i.e.}, the same hypotheses are in $\closure(\hypclass^t)$.
  \item {\bf C2:} The response is given as $y^t=0$. In this case $\widetilde{L}^{t}$ is updated to $\widetilde{L}^{t} = c^t$ in Algorithm~\ref{alg.learnhm} (lines 9--12). Hence, $\hypclass(\widetilde{L}^{t}, \widetilde{U}^{t})$ consists of all hypothesis $\hyp \in \hypclass({L}^{t-1}, {U}^{t-1}) =  \closure( \hypclass^{t-1} )$ that satisfy $c^t \leq \hyp_{i^t,j^t}$. Also, $\closure(\hypclass^t) = \closure(\hyp^{t-1}) \cap \{ \hyp \in \R^{\num \times \num} \mid c^{t'} \leq \hyp_{i^{t'}, j^{t'}} \}$, \emph{i.e.}, the same hypotheses are in $\closure(\hypclass^t)$.
\end{itemize}
\end{proof}
\begin{lemma} \label{lemma.projection.valid2}
Problem~\ref{sec.algoPart2-LUBounds.proj.5} has a unique optimal solution $U^{*t}$ that is given by 
$U^{*t}_{i,j} = \max_{\hyp \in \hypclass(\widetilde{L}^{t}, \widetilde{U}^{t})} \hyp_{i,j}$ and satisfies $\hypclass(\widetilde{L}^{t}, U^{*t}) = \hypclass(\widetilde{L}^{t}, \widetilde{U}^{t})$.
\end{lemma}
\begin{proof} 
Let us define $\widetilde{\hyp}_{\max}^t \in \R^{\num \times \num}$ via
\begin{align*}
  \widetilde{\hyp}_{\max;i, j}^t = \max_{\hyp \in \hypclass(\widetilde{L}^t, \widetilde{U}^t)} \hyp_{i,j} \quad \forall i,j \in [\num].
\end{align*}
Note that $\widetilde{\hyp}_{\max}^t \in \R^{\num \times \num}$ is feasible for Problem~\ref{sec.algoPart2-LUBounds.proj.5}:
\begin{itemize}
  \item $\widetilde{\hyp}_{\max}^t \leq \widetilde{U}^t$ is obvious from the definition.
  \item To see that $\widetilde{\hyp}_{\max}^t \in \hypclass$, fix $i,j \in [\num]$. Then, there exists $\hyp \in \hypclass$ such that $\widetilde{\hyp}_{\max;i,j}^t = \hyp_{i,j}$. Since $\hyp$ is a hemimetric, we have $\widetilde{\hyp}_{\max;i,j}^t = \hyp_{i,j} \leq \hyp_{i,k} + \hyp_{k,j} \leq \widetilde{\hyp}_{\max;i,k}^t + \widetilde{\hyp}_{\max;k,j}^t$.
\end{itemize}

We now show that $U^{*t}_{i,j} = \max_{\hyp \in \hypclass(\widetilde{L}^{t}, \widetilde{U}^{t})} \hyp_{i,j}$, which also implies uniqueness of the optimal solution. Consider the following two cases:
\begin{itemize}
  \item Assume $U^{*t} \leq \widetilde{\hyp}_{\max}^t$. If $U^{*t} = \widetilde{\hyp}_{\max}^t$, then we are done. Otherwise, this violates optimality of $U^{*t}$ as $\widetilde{\hyp}_{\max}^t$ is a better solution. 
  \item Otherwise, assume $\exists i,j \in [\num] \colon U^{*t}_{i,j} > \widetilde{\hyp}_{\max;i,j}^t$. To show a contradiction, define $M \in \R^{\num \times \num}$ via
\begin{align}
  M_{i,j} = \max \left({{U}^{*t}_{i,j}, \widetilde{\hyp}_{\max;i,j}^t}\right) \quad \forall i,j \in [\num].
\end{align}
The matrix $M$ has strictly smaller objective value for Problem~\ref{sec.algoPart2-LUBounds.proj.5} than $U^{*t}$. By showing that $M$ is feasible for Problem~\ref{sec.algoPart2-LUBounds.proj.5}, we reach a contradiction to the optimality of $U^{*t}$. Clearly, $M \leq \widetilde{U}^t$. 
To see that $M \in \hypclass$ consider the following two cases for $i,j \in [\num]$:
\begin{itemize}
  \item Case $M_{i,j} = {U}^{*t}_{i,j}$. From this we get
\begin{align*}
  M_{i,j} = {U}^{*t}_{i,j} \leq {U}^{*t}_{i,k} + {U}^{*t}_{k,j} \leq M_{i,k} + M_{k,j},
\end{align*}
where the first inequality holds because $U^{*t}$ is a hemimetric and the second inequality holds by definition of $M$.
  \item Case $M_{i,j} = \widetilde{\hyp}_{\max;i,j}^t$. Here we have
\begin{align*}
  M_{i,j} = \widetilde{\hyp}_{\max;i,j}^t \leq \widetilde{\hyp}_{\max;i,k}^t + \widetilde{\hyp}_{\max;k,j}^t \leq M_{i,k} + M_{k,j},
\end{align*}
where the first inequality holds because $\widetilde{\hyp}_{\max}$ is a hemimetric and the second inequality again holds by definition of $M$. 
\end{itemize}
 This shows that $M$ is feasible --- we reach a contradiction, and we must have $U^{*t} = \widetilde{\hyp}_{\max}^t$.
\end{itemize}

Next, we show that $\hypclass(\widetilde{L}^{t}, U^{*t}) = \hypclass(\widetilde{L}^{t}, \widetilde{U}^{t})$. Observe that $\hypclass(\widetilde{L}^{t}, U^{*t}) \subseteq \hypclass(\widetilde{L}^{t}, \widetilde{U}^{t})$ is obvious. On the other hand, any metric $\hyp \in \hypclass(\widetilde{L}^{t}, \widetilde{U}^{t})$ is also in $\hypclass(\widetilde{L}^{t}, U^{*t})$ because $U^{*t} = \widetilde{\hyp}_{\max}^t$ is by definition element-wise larger than any metric in $\hypclass(\widetilde{L}^{t}, \widetilde{U}^{t},)$ . Hence, $\hypclass(\widetilde{L}^{t}, U^{*t}) = \hypclass(\widetilde{L}^{t}, \widetilde{U}^{t})$.
Informally, this implies that replacing the upper bound $\widetilde{U}^{t}$ on the solution space by ${U}^{*t}$ does not remove any hypotheses.
\end{proof}

\begin{lemma} \label{lemma.projection.valid3}
Problem~\ref{sec.algoPart2-LUBounds.proj.6} has a unique solution $L^{*t}$ that is given by 
$L^{*t}_{i,j} = \min_{\hyp \in \hypclass(\widetilde{L}^{t}, {U}^{*t})} \hyp_{i,j}$ and satisfies $\hypclass({L}^{*t}, U^{*t}) = \hypclass(\widetilde{L}^{t}, U^{*t})$.
\end{lemma}
\begin{proof}
This proof essentially follows the lines of the proof of Lemma~\ref{lemma.projection.valid2}.
Let us define $\widetilde{\hyp}_{\min}^t \in \R^{\num \times \num}$ via
\begin{align*}
  \widetilde{\hyp}_{\min;i, j}^t = \min_{\hyp \in \hypclass(\widetilde{L}^t, {U}^{*t})} \hyp_{i,j} \quad \forall i,j \in [\num].
\end{align*}
Note that $\widetilde{\hyp}_{\min}^t \in \R^{\num \times \num}$ is feasible for Problem~\ref{sec.algoPart2-LUBounds.proj.6}:
\begin{itemize}
  \item $\widetilde{\hyp}_{\min}^t \geq \widetilde{L}^t$ is obvious from the definition.
  \item To see that $\widetilde{\hyp}_{\min}^t \in \lclass(U^{t*})$, fix $i,j \in [\num]$. Then, there exists $\hyp \in \hypclass(\widetilde{L}^t, {U}^{*t})$ such that $\widetilde{\hyp}_{\min;i,j}^t = \hyp_{i,j}$. Since $\hyp$ is a hemimetric, we have $\widetilde{\hyp}_{\min;i,j}^t = \hyp_{i,j} \geq \hyp_{i,k} - \hyp_{j,k} \geq \widetilde{\hyp}_{\min;i,k}^t - {U}^{*t}_{j,k}$ since $\hyp_{j,k} \leq {U}^{*t}_{j,k}$. ($\widetilde{\hyp}_{\min;i,j}^t \geq \widetilde{\hyp}_{\min;k,j}^t - {U}^{*t}_{k,i}$ follows analogously)
\end{itemize}
We now show that $L^{*t}_{i,j} = \min_{\hyp \in \hypclass(\widetilde{L}^{t}, {U}^{*t})} \hyp_{i,j}$, which also implies uniqueness of the optimal solution.
Consider the following two cases:
\begin{itemize}
  \item Assume $L^{*t} \geq \widetilde{\hyp}_{\min}^t$. If $L^{*t} = \widetilde{\hyp}_{\min}^t$, then we are done. Otherwise this violates the optimality of $L^{*t}$, as $\widetilde{\hyp}_{\min}^t$ is a better solution.
  \item Otherwise, assume $\exists i,j \in [\num]\colon L^{*t}_{i,j} < \widetilde{\hyp}_{\min;i,j}^t$. To show a contradiction, define $M \in \R^{\num \times \num}$ via
  \begin{align*}
    M_{i,j} = \min \left({{L}^{*t}_{i,j}, \widetilde{\hyp}_{\min;i,j}^t}\right) \quad \forall i,j \in [\num].
  \end{align*}
  The matrix $M$ has strictly smaller objective value for Problem~\ref{sec.algoPart2-LUBounds.proj.6} than $L^{*t}$. By showing that $M$ is feasible for Problem~\ref{sec.algoPart2-LUBounds.proj.6}, we reach a contradiction to the optimality of $L^{*t}$. Clearly, $M \geq \widetilde{L}^t$. To see that $M \in \lclass(U^{t*})$ consider the following two cases for $i,j \in [\num]$:
    \begin{itemize}
      \item Case $M_{i,j} = {L}^{*t}_{i,j}$. As ${L}^{*t}_{i,j} \in \lclass(U^{*t})$, we have that
        \begin{align*}
          M_{i,j} = {L}^{*t}_{i,j} \geq {L}^{*t}_{i,k} - U^{*t}_{j,k} \geq M_{i,k} - U^{*t}_{j,k},
        \end{align*} 
        where the first second inequality follows from the definition of $M$. ($M_{i,j} \geq M_{k,j} - {U}^{*t}_{k,i}$ follows analogously)
      \item Case $M_{i,j} = \widetilde{\hyp}_{\min;i,j}^t$. From this we get that there is some $D \in \hypclass(\widetilde{L}^t, {U}^{*t})$ such that $\hyp_{i,j} = \widetilde{\hyp}_{\min;i,j}^t$. As above, this gives 
      \begin{align*}
        M_{i,j} = \hyp_{i,j} \geq \hyp_{i,k} - \hyp_{j,k} \geq M_{i,k} - {U}^{*t}_{j,k},
      \end{align*}
      where the first inequality follows because $\hyp$ is a hemimetric and the seconds inequality follows by the definition of $M$ and the fact that $\hyp_{j,k} \leq {U}^{*t}_{j,k}$. ($M_{i,j} \geq M_{k,j} - {U}^{*t}_{k,i}$ follows analogously)
    \end{itemize}
     This shows that $M$ is feasible --- we reach a contradiction, and we must have $L^{*t} = \widetilde{\hyp}_{\min}^t$.
\end{itemize}

Next we show that $\hypclass({L}^{*t}, U^{*t}) = \hypclass(\widetilde{L}^{t}, {U}^{*t})$. Observe that $\hypclass({L}^{*t}, U^{*t}) \subseteq \hypclass(\widetilde{L}^{t}, {U}^{*t})$ is obvious. 
On the other hand, any metric $\hyp \in \hypclass(\widetilde{L}^{t}, {U}^{*t})$ is also in $\hypclass({L}^{*t}, U^{*t})$ because $L^{*t} = \widetilde{\hyp}_{\min}^t$ is by definition element-wise smaller than any metric in $\hypclass(\widetilde{L}^{t}, {U}^{*t},)$ . Hence, $\hypclass({L}^{*t}, U^{*t}) = \hypclass(\widetilde{L}^{t}, {U}^{*t})$.
\end{proof}

We can now prove the theorem.
\begin{proof}[\bfseries{Proof of Theorem~\ref{thm.projection}}]
  Let $L',U'$ be an optimal solution of Problem~\ref{eq.jointprojection}. Thus $\hypclass(L',U') \supseteq \hypclass^t$.
  
  Define $\hyp_{\min}$ and $\hyp_{\max}$ as the element-wise minimum and maximum of all $\hyp \in \closure(\hypclass^t)$, respectively. Thus, by using Lemmas~\ref{lemma.projection.valid1},~\ref{lemma.projection.valid2} and~\ref{lemma.projection.valid3} we have that $L^{*t} = \hyp_{\min}$ and $U^{*t} = \hyp_{\max}$. Consequently, $\hypclass(L^{*t}, U^{*t}) = \closure(\hypclass^t) \supseteq \hypclass^t$, \emph{i.e.}, $L^{*t}, U^{*t}$ are feasible for Problem~\ref{eq.jointprojection}.

We now show that $L^{*t} = L'$ and $U^{*t} = U'$. If $L^{*t} = L'$ and $U^{*t} = U'$, then we are done. Otherwise, $L^{*t} \neq L'$ or $U^{*t} \neq U'$ contradicts either the optimality or feasibility of $L', U'$ --- consider the following two cases:
\begin{itemize}
  \item Let $U^{*t} \neq U'$.
    \begin{itemize}
	  \item If $U^{*t} \leq U'$ and $\exists i,j \in [\num]: U^{*t}_{i,j} < U'_{i,j}$, then $U'$ cannot be optimal because $L', U^{*t}$ is a feasible solution with smaller objective value.
	  \item If on the other hand, $\exists i,j \in [\num]: U^{*t}_{i,j} > U'_{i,j}$, then $U'$ cannot be feasible because it \emph{rules out} all hemimetrics that coincide with $U^{*t}$ on the $(i,j)$th value.
    \end{itemize}
    
  \item For the case $L^{*t} \neq L'$, similar arguments as above hold.
\end{itemize}
Hence, $L'=L^{*t}$, $U'=U^{*t}$ is the unique optimal solution for Problem~\ref{eq.jointprojection}.
\end{proof}

%
%
%



%

\section{Proof of Theorem~\ref{thm.luprojalgo}} \label{app.thm.luprojalgo}

The proof of Theorem~\ref{thm.luprojalgo} builds upon Lemma~\ref{lemma.luprojalgo.convergeU} and Lemma~\ref{lemma.luprojalgo.convergeL}.

$\uproj$ is in fact equivalent to the Floyd-Warshall algorithm for solving the all-pair shortest paths problem in a graph~\cite{floyd1962algorithm}. Similar equivalence has been shown by~\citet{brickell2008metric} while studying the problem of projecting a non-metric matrix to a metric via decrease-only projections. \citet{brickell2008metric} used a similar result as Lemma~\ref{lemma.luprojalgo.convergeU}, however they directly used the interpretation of shortest-path problem. We will give an algebric proof for completeness which also helps us develop similar proof for Lemma~\ref{lemma.luprojalgo.convergeL}

\begin{lemma} \label{lemma.luprojalgo.convergeU}
The function $\uproj$ in Algorithm~\ref{alg.uproj} with input parameters $\widetilde{U}^t$ returns $U^t \in \uclass$ with the following property:  for all $i,j\in[\num]$ it holds that $\big( \textnormal{\texttt{Either}}$ $U^t_{i,j} = \widetilde{U}^t_{i,j}$; $\textnormal{\texttt{OR}}$ \ $\exists k \textnormal{ s.t. }  U^t_{i,j} = U^t_{i,k} + U^t_{k,j} \big) $.
\end{lemma}

\begin{proof}[\bfseries{Proof of Lemma~\ref{lemma.luprojalgo.convergeU}}]
For ease of notation, let us define $M^0 = \widetilde{U}^t$. During an iteration of the outer loop of \emph{pivots} $k \in [\num]$ for \uproj, we update $M^k_{i, j} = \min {\left(M^{k-1}_{i, j}, M^{k-1}_{i, k} + M^{k-1}_{j, k}\right)}$. At the end of execution, we have $k=\num$, and the output of \uproj is $M^\num$. We will prove the following properties of $M^k$ at the end of every iteration $k \in [\num]$ and $\forall a, b \in [\num]$:
\begin{align}
&\forall k' \in [k] \ M^{k}_{a,b} \leq  M^{k}_{a, k'} + M^{k}_{k', b} \label{lemma.luprojalgo.convergeU.eq1} \\
&\texttt{Either } M^{k}_{a,b} = M^{0}_{a,b}  \texttt{ OR } \exists k' \in [k] \textnormal{ s.t. } M^k_{a,b} = M^k_{a,k'} + M^k_{k',b} \label{lemma.luprojalgo.convergeU.eq2} 
\end{align}

Note that the conditions hold trivially at the start since the range of $k'$ is $\emptyset$. Also, the claim of the lemma is equivalent to the statement that the two conditions above hold at the end when $k = \num$. We will now prove this by induction.

{\bfseries Base case for $M^0$, $M^1$:}\\
Conditions in Equations~\ref{lemma.luprojalgo.convergeU.eq1},~\ref{lemma.luprojalgo.convergeU.eq2} hold trivially for $M^0$ since the range of $k'$ is $\emptyset$. Conditions in Equations~\ref{lemma.luprojalgo.convergeU.eq1},\ref{lemma.luprojalgo.convergeU.eq2} hold for $M^1$ since the range of $k'$ is $[1]$, and $\forall a, b \in [n] \ \text{ update } M^1_{a, b} = \min {\left(M^{0}_{a, b}, M^{0}_{a, 1} + M^{0}_{1, b}\right)}$ ensures that the conditions hold.

{\bfseries Induction hypothesis:}\\
At $k-1$, $M^{k-1}$ satisfies the conditions in Equations~\ref{lemma.luprojalgo.convergeU.eq1},~\ref{lemma.luprojalgo.convergeU.eq2}.

{\bfseries Inductive step:}\\
Prove that $M^{k}$ satisfies the conditions in Equations~\ref{lemma.luprojalgo.convergeU.eq1},~\ref{lemma.luprojalgo.convergeU.eq2} after we make the following update at iteration $k$:
\begin{align*}
\forall a, b \in [\num] \ \text{ update } M^k_{a, b} = \min {\left(M^{k-1}_{a, b}, M^{k-1}_{a, k} + M^{k-1}_{k, b}\right)}
\end{align*}

Consider any $a,b \in [n]$. We consider the following $3$ cases depending on the type of update to $M^k_{a, b}$.
\begin{enumerate}[label={\bf C{\arabic*}}:]
\item $M^k_{a, b} =  M^{k-1}_{a, k} + M^{k-1}_{k, b}$
\item $M^k_{a, b} =  M^{k-1}_{a, b}$
	\begin{enumerate}[label={\bf C{\arabic{enumi}.\arabic*}}:] 
		\item $M^{k-1}_{a, b} = M^{0}_{a, b}$
		\item $M^{k-1}_{a, b} = M^{k-1}_{a, e} + M^{k-1}_{e, b} \text{ where } e \in [k-1]$
	\end{enumerate}
\end{enumerate}
Note that the cases {\bf C1}, {\bf C2.1} and {\bf C2.2} are exhaustive given the update rule at $k$ and the induction hypothesis for $M^{k-1}$.

{\bfseries Proving that the condition of Equation~\ref{lemma.luprojalgo.convergeU.eq1} holds for $M^k$}.\\
Let us first prove that the condition of Equation~\ref{lemma.luprojalgo.convergeU.eq1} holds for $M^k$. The proof is same for all the three cases {\bf C1}, {\bf C2.1} and {\bf C2.2}.

For $k'=k$. Then,
\begin{align*}
M^{k}_{a, k'} + M^{k}_{k',b} &= M^{k-1}_{a, k} + M^{k-1}_{k,b} \\
							    &\geq \min {\left(M^{k-1}_{a, b}, M^{k-1}_{a, k} + M^{k-1}_{k, b}\right)} \\
							    &= M^{k}_{a, b}
\end{align*}

For $k'\in [k], k' \neq k$. If, after the update, we have $M^{k}_{a, k'} = M^{k-1}_{a, k'}$ and $M^{k}_{k',b} = M^{k-1}_{k',b}$. Then, 
\begin{align*}
M^{k}_{a, k'} + M^{k}_{k',b} &= M^{k-1}_{a, k'} + M^{k-1}_{k',b} \\
							    &\geq M^{k-1}_{a, b} \\
							    &\geq \min {\left(M^{k-1}_{a, b}, M^{k-1}_{a, k} + M^{k-1}_{k, b}\right)} \\
							    &= M^{k}_{a, b}
\end{align*}

For $k'\in [k], k' \neq k$. If, after the update, we have $M^{k}_{a, k'} = M^{k-1}_{a, k'}$ and $M^{k}_{k',b} = M^{k-1}_{k',k} + M^{k-1}_{k,b}$. Then, 
\begin{align*}
M^{k}_{a, k'} + M^{k}_{k',b} &= M^{k-1}_{a, k'} + M^{k-1}_{k',k} + M^{k-1}_{k,b} \\
								&\geq M^{k-1}_{a, k} + M^{k-1}_{k,b}\\ 
                                 &\geq \min {\left(M^{k-1}_{a, b}, M^{k-1}_{a, k} + M^{k-1}_{k, b}\right)} \\
                                 &= M^{k}_{a, b}
\end{align*}

For $k'\in [k], k' \neq k$. If, after the update, we have $M^{k}_{a, k'} = M^{k-1}_{a,k} + M^{k-1}_{k,k'}$ and $M^{k}_{k',b} = M^{k-1}_{k',b}$. Then, 
\begin{align*}
M^{k}_{a, k'} + M^{k}_{k',b} &= M^{k-1}_{a,k} + M^{k-1}_{k,k'} + M^{k-1}_{k',b} \\
								&\geq M^{k-1}_{a, k} + M^{k-1}_{k,b}\\ 
                                 &\geq \min {\left(M^{k-1}_{a, b}, M^{k-1}_{a, k} + M^{k-1}_{k, b}\right)} \\
                                 &= M^{k}_{a, b}
\end{align*}

For $k'\in [k], k' \neq k$. If, after the update, we have $M^{k}_{a, k'} = M^{k-1}_{a,k} + M^{k-1}_{k,k'}$ and $M^{k}_{k',b} =  M^{k-1}_{k',k} + M^{k-1}_{k,b}$. Then, 
\begin{align*}
M^{k}_{a, k'} + M^{k}_{k',b} &= M^{k-1}_{a,k} + M^{k-1}_{k,k'} + M^{k-1}_{k',k} + M^{k-1}_{k,b}  \\
								&\geq M^{k-1}_{a, k} + M^{k-1}_{k,b}\\ 
                                 &\geq \min {\left(M^{k-1}_{a, b}, M^{k-1}_{a, k} + M^{k-1}_{k, b}\right)} \\
                                 &= M^{k}_{a, b}
\end{align*}

{\bfseries Proving that the condition of Equation~\ref{lemma.luprojalgo.convergeU.eq2} holds for $M^k$}.\\
We will do this analysis case by case for  the three cases {\bf C1}, {\bf C2.1} and {\bf C2.2}.

For case {\bf C1}, \emph{i.e.} $M^k_{a, b} =  M^{k-1}_{a, k} + M^{k-1}_{k, b}$. Then, 
\begin{align*}
M^k_{a, b}  &=  M^{k-1}_{a, k} + M^{k-1}_{k, b} \\
			 &= M^{k}_{a, k} + M^{k}_{k, b}
\end{align*}
Hence the condition of Equation~\ref{lemma.luprojalgo.convergeU.eq2} holds.

For case {\bf C2.1}, \emph{i.e.} $M^k_{a, b} =  M^{k-1}_{a, b} = M^{0}_{a, b}$. Then, 
\begin{align*}
M^k_{a, b} &=  M^{k-1}_{a, b} \\
             &= M^{0}_{a, b}
\end{align*}
Hence the condition of Equation~\ref{lemma.luprojalgo.convergeU.eq2} holds.

For case {\bf C2.2}, \emph{i.e.} $M^k_{a, b} =  M^{k-1}_{a, b} = M^{k-1}_{a, e} + M^{k-1}_{e, b} \text{ where } e \in [k-1]$. Then, if $M^{k-1}_{a, e} + M^{k-1}_{e, b} = M^{k}_{a, e} + M^{k}_{e, b}$, we have:
\begin{align*}
M^k_{a, b}  &=  M^{k-1}_{a, e} + M^{k-1}_{e, b} \\
			 &= M^{k}_{a, e} + M^{k}_{e, b}
\end{align*}
Hence, assuming $M^{k-1}_{a, e} + M^{k-1}_{e, b} = M^{k}_{a, e} + M^{k}_{e, b}$, the condition of Equation~\ref{lemma.luprojalgo.convergeU.eq2} holds. 
We conclude by showing that $M^{k-1}_{a, e} + M^{k-1}_{e, b} = M^{k}_{a, e} + M^{k}_{e, b}$ holds for {\bf C2.2}.
We show this by contradiction as follows. 
For the sake of contradiction, let us assume that $M^{k-1}_{a, e} + M^{k-1}_{e, b} > M^{k}_{a, e} + M^{k}_{e, b}$. 
Then, 
\begin{align*}
M^k_{a, b} = M^{k-1}_{a, b} &= M^{k-1}_{a, e} + M^{k-1}_{e, b} \\
 							   &> M^{k}_{a, e} + M^{k}_{e, b} \\
\implies M^k_{a, b} &> M^{k}_{a, e} + M^{k}_{e, b} 
\end{align*}
This is a contradiction to the proof we did above showing that $M^k$ satisfies Equation~\ref{lemma.luprojalgo.convergeU.eq1}. Hence, it holds that $M^k_{a, b}  = M^{k}_{a, e} + M^{k}_{e, b}$.

Now, putting this all together, we have proved that $M^k$ satisfies the conditions of Equations~\ref{lemma.luprojalgo.convergeU.eq1},~\ref{lemma.luprojalgo.convergeU.eq2}. Hence by induction, the final output $M^n$ of the algorithm \uproj satisfies these conditions as well, which is exactly equivalent to the claim of the lemma we want to prove.
\end{proof}


\begin{lemma} \label{lemma.luprojalgo.convergeL}
The function $\lproj$ in Algorithm~\ref{alg.lproj} with input parameters $\widetilde{L}^t, U^t$  returns $L^t \in \lclass(U^t)$ with the following property: for all $i,j\in[\num]$ it holds that  $\big( \textnormal{\texttt{Either}}$ $L^t_{i,j} = \widetilde{L}^t_{i,j}$; $\textnormal{\texttt{OR}}$ \ $\exists k \textnormal{ s.t. } L^{t}_{i, j} = \max {\big(L^{t}_{i, k} - U^{t}_{j, k}, L^{t}_{k, j} - U^{t}_{k, i}\big)} \big)$.
\end{lemma}

\begin{proof}[\bfseries{Proof of Lemma~\ref{lemma.luprojalgo.convergeL}}]
The proof follows similar ideas as in Lemma~\ref{lemma.luprojalgo.convergeU}, however requires much more algebraic manipulations and cases for complete analysis. Furthermore, there is no direct analogue to shortest-path problem that existed for Lemma~\ref{lemma.luprojalgo.convergeU}.

For ease of notation, let us define $M^0 = \widetilde{L}^t$. And, during an iteration of the outer loop of \emph{pivots} $k \in [\num]$ for \lproj, we update
$$M^k_{i, j} = \max {\left(M^{k-1}_{i, j}, M^{k-1}_{i, k} - U^{t}_{j, k}, M^{k-1}_{k, j} - U^{t}_{k, i}\right)}$$
At the end of execution, we have $k=\num$, and the output of \lproj is $M^\num$. We will prove following properties of $M^k$ at the end of every iteration $k \in [\num]$ for $\forall a, b \in [\num]$:
\begin{align}
&\forall k' \in [k] \ M^{k}_{a,b} \geq  \max \left(M^{k}_{a, k'} - U^{t}_{b, k'}, M^{k}_{k', b} - U^{t}_{k', a}\right)  \label{lemma.luprojalgo.convergeL.eq1} \\
&\texttt{Either } M^{k}_{a,b} = M^{0}_{a,b}  \texttt{ OR } \exists k' \in [k] \text{ s.t. } M^k_{a,b} = \max \left(M^{k}_{a, k'} - U^{t}_{b, k'}, M^{k}_{k', b} - U^{t}_{k', a}\right) \label{lemma.luprojalgo.convergeL.eq2} 
\end{align}

Note that the conditions hold trivially at the start since the range of $k'$ is $\emptyset$. Also, the claim of the lemma is equivalent to the statement that the two conditions above hold at the end when $k = \num$. We will now prove this by induction.

{\bfseries Base case for $M^0$, $M^1$:}\\
Conditions in Equations~\ref{lemma.luprojalgo.convergeL.eq1},~\ref{lemma.luprojalgo.convergeL.eq2} hold trivially for $M^0$ since the range of $k'$ is $\emptyset$. Conditions in Equations~\ref{lemma.luprojalgo.convergeL.eq1},\ref{lemma.luprojalgo.convergeL.eq2} hold for $M^1$ since the range of $k'$ is $[1]$, and $\forall a, b \in [n] \ \text{ update } M^1_{a, b} = \max {\left(M^{0}_{a, b}, M^{0}_{a, 1} - U^{t}_{b, 1}, M^{0}_{1, a} - U^{t}_{1, b}\right)} $ ensures that conditions hold.

{\bfseries Induction hypothesis:}\\
At $k-1$, $M^{k-1}$ satisfies the conditions in Equations~\ref{lemma.luprojalgo.convergeL.eq1},~\ref{lemma.luprojalgo.convergeL.eq2}.

{\bfseries Inductive step:}\\
Prove that $M^{k}$ satisfies the conditions in Equations~\ref{lemma.luprojalgo.convergeL.eq1},~\ref{lemma.luprojalgo.convergeL.eq2}
after we make the following update at iteration $k$:
\begin{align*}
\forall a, b \in [\num] \ \text{ update } M^k_{a, b} = \max {\left(M^{k-1}_{a, b}, M^{k-1}_{a, k} - U^{t}_{b, k}, M^{k-1}_{k, b} - U^{t}_{k, a}\right)}
\end{align*}

Consider any $a,b \in [n]$. We consider the following $3$ cases depending on the type of update to $M^k_{a, b}$.
\begin{enumerate}[label={\bf C{\arabic*}}:]
\item $M^k_{a, b} =  \max {\left(M^{k-1}_{a, k} - U^{t}_{b, k}, M^{k-1}_{k, b} - U^{t}_{k, a}\right)}$
\item $M^k_{a, b} =  M^{k-1}_{a, b}$
	\begin{enumerate}[label={\bf C{\arabic{enumi}.\arabic*}}:] 
		\item $M^{k-1}_{a, b} = M^{0}_{a, b}$
		\item $M^{k-1}_{a, b} = \max {\left(M^{k-1}_{a, e} - U^{t}_{b, e}, M^{k-1}_{e, b} - U^{t}_{e, a}\right)} \text{ where } e \in [k-1]$
	\end{enumerate}
\end{enumerate}
Note that the cases {\bf C1}, {\bf C2.1} and {\bf C2.2} are exhaustive given the update rule at $k$ and the induction hypothesis for $M^{k-1}$.

{\bfseries Proving that the condition of Equation~\ref{lemma.luprojalgo.convergeL.eq1} holds for $M^k$.}\\
Let us first prove that the condition of Equation~\ref{lemma.luprojalgo.convergeL.eq1} holds for $M^k$. The proof is same for all the three cases {\bf C1}, {\bf C2.1} and {\bf C2.2}.

For $k'=k$. Then,
\begin{align*}
\max {\left(M^{k}_{a, k'} - U^{t}_{b, k'}, M^{k}_{k', b} - U^{t}_{k', a}\right)} &= \max {\left(M^{k-1}_{a, k} - U^{t}_{b, k}, M^{k-1}_{k, b} - U^{t}_{k, a}\right)} \\
&\leq \max {\left(M^{k-1}_{a, b}, M^{k-1}_{a, k} - U^{t}_{b, k}, M^{k-1}_{k, b} - U^{t}_{k, a}\right)} \\
&= M^{k}_{a, b}
\end{align*}

For $k'\in [k], k' \neq k$. If, after update, we have $M^{k}_{a, k'} = M^{k-1}_{a, k'}$ and $M^{k}_{k',b} = M^{k-1}_{k',b}$. Then, 
\begin{align*}
\max {\left(M^{k}_{a, k'} - U^{t}_{b, k'}, M^{k}_{k', b} - U^{t}_{k', a}\right)} &= \max {\left(M^{k-1}_{a, k'} - U^{t}_{b, k'}, M^{k-1}_{k', b} - U^{t}_{k', a}\right)} \\
&\leq M^{k-1}_{a, b} \\
&\leq \max {\left(M^{k-1}_{a, b}, M^{k-1}_{a, k} - U^{t}_{b, k}, M^{k-1}_{k, b} - U^{t}_{k, a}\right)} \\
&= M^{k}_{a, b}
\end{align*}

For $k'\in [k], k' \neq k$. If, after update, we have $M^{k}_{a, k'} = M^{k-1}_{a, k'}$ and $M^{k}_{k',b} = M^{k-1}_{k',k} - U^{t}_{b,k}$. Then, 
\begin{align*}
\max {\left(M^{k}_{a, k'} - U^{t}_{b, k'}, M^{k}_{k', b} - U^{t}_{k', a}\right)} &= \max {\left(M^{k-1}_{a, k'} - U^{t}_{b, k'},  M^{k-1}_{k',k} - U^{t}_{b,k} - U^{t}_{k', a}\right)} \\
&= \max {\left(M^{k-1}_{a, k'} - U^{t}_{b, k'},  M^{k-1}_{k',k} - U^{t}_{k', a} - U^{t}_{b,k}\right)} \\
&\leq \max {\left(M^{k-1}_{a, b},  M^{k-1}_{a,k} - U^{t}_{b,k}\right)} \\
&\leq \max {\left(M^{k-1}_{a, b}, M^{k-1}_{a, k} - U^{t}_{b, k}, M^{k-1}_{k, b} - U^{t}_{k, a}\right)} \\
&= M^{k}_{a, b}
\end{align*}

For $k'\in [k], k' \neq k$. If, after update, we have $M^{k}_{a, k'} = M^{k-1}_{a, k'}$ and $M^{k}_{k',b} = M^{k-1}_{k, b} - U^{t}_{k, k'}$. Then, 
\begin{align*}
\max {\left(M^{k}_{a, k'} - U^{t}_{b, k'}, M^{k}_{k', b} - U^{t}_{k', a}\right)} &= \max {\left(M^{k-1}_{a, k'} - U^{t}_{b, k'},  M^{k-1}_{k, b} - U^{t}_{k, k'} - U^{t}_{k', a}\right)} \\
&\leq \max {\left(M^{k-1}_{a, k'} - U^{t}_{b, k'},  M^{k-1}_{k, b} - U^{t}_{k, a}\right)} \\
&\leq \max {\left(M^{k-1}_{a, b},  M^{k-1}_{k, b} - U^{t}_{k, a}\right)} \\
&\leq \max {\left(M^{k-1}_{a, b}, M^{k-1}_{a, k} - U^{t}_{b, k}, M^{k-1}_{k, b} - U^{t}_{k, a}\right)} \\
&= M^{k}_{a, b}
\end{align*}

For $k'\in [k], k' \neq k$. If, after update, we have $M^{k}_{a, k'} = M^{k-1}_{a, k} - U^{t}_{k', k}$ and $M^{k}_{k',b} = M^{k-1}_{k',b}$. Then, 
\begin{align*}
\max {\left(M^{k}_{a, k'} - U^{t}_{b, k'}, M^{k}_{k', b} - U^{t}_{k', a}\right)} &=\max {\left(M^{k-1}_{a, k} - U^{t}_{k', k} - U^{t}_{b, k'}, M^{k-1}_{k',b} - U^{t}_{k', a}\right)} \\
&\leq\max {\left(M^{k-1}_{a, k} - U^{t}_{b, k}, M^{k-1}_{k',b} - U^{t}_{k', a}\right)} \\
&\leq \max {\left(M^{k-1}_{a,k} - U^{t}_{b,k}, M^{k-1}_{a, b}\right)} \\
&\leq \max {\left(M^{k-1}_{a, b}, M^{k-1}_{a, k} - U^{t}_{b, k}, M^{k-1}_{k, b} - U^{t}_{k, a}\right)} \\
&= M^{k}_{a, b}
\end{align*}

For $k'\in [k], k' \neq k$. If, after update, we have $M^{k}_{a, k'} = M^{k-1}_{k, k'} - U^{t}_{k, a}$ and $M^{k}_{k',b} = M^{k-1}_{k',b}$. Then, 
\begin{align*}
\max {\left(M^{k}_{a, k'} - U^{t}_{b, k'}, M^{k}_{k', b} - U^{t}_{k', a}\right)} &=\max {\left(M^{k-1}_{k, k'} - U^{t}_{k, a} - U^{t}_{b, k'}, M^{k-1}_{k',b} - U^{t}_{k', a}\right)} \\
&=\max {\left(M^{k-1}_{k, k'} - U^{t}_{b, k'} - U^{t}_{k, a}, M^{k-1}_{k',b} - U^{t}_{k', a}\right)} \\
&\leq\max {\left(M^{k-1}_{k, b} - U^{t}_{k, a}, M^{k-1}_{a, b}\right)} \\
&\leq \max {\left(M^{k-1}_{a, b}, M^{k-1}_{a, k} - U^{t}_{b, k}, M^{k-1}_{k, b} - U^{t}_{k, a}\right)} \\
&= M^{k}_{a, b}
\end{align*}

For $k'\in [k], k' \neq k$. We need to consider following four more possibilities after update:
\begin{itemize}
\item $M^{k}_{a, k'} = M^{k-1}_{a, k} - U^{t}_{k', k}$ and $M^{k}_{k',b} = M^{k-1}_{k',k} - U^{t}_{b,k}$
\item $M^{k}_{a, k'} = M^{k-1}_{a, k} - U^{t}_{k', k}$ and $M^{k}_{k',b} = M^{k-1}_{k, b} - U^{t}_{k, k'}$
\item $M^{k}_{a, k'} = M^{k-1}_{k, k'} - U^{t}_{k, a}$ and $M^{k}_{k',b} = M^{k-1}_{k',k} - U^{t}_{b,k}$
\item $M^{k}_{a, k'} = M^{k-1}_{k, k'} - U^{t}_{k, a}$ and $M^{k}_{k',b} = M^{k-1}_{k, b} - U^{t}_{k, k'}$
\end{itemize}
These four possibilities follow exactly the same arguments as used above.

{\bfseries Proving that the condition of Equation~\ref{lemma.luprojalgo.convergeL.eq2} holds for $M^k$}.\\
We will do this analysis case by case for the three cases {\bf C1}, {\bf C2.1} and {\bf C2.2}. 

For case {\bf C1}, \emph{i.e.} $M^k_{a, b} =  \max {\left(M^{k-1}_{a, k} - U^{t}_{b, k}, M^{k-1}_{k, b} - U^{t}_{k, a}\right)}$. Then, 
\begin{align*}
M^k_{a, b}  &=  \max {\left(M^{k-1}_{a, k} - U^{t}_{b, k}, M^{k-1}_{k, b} - U^{t}_{k, a}\right)} \\
			 &= \max {\left(M^{k}_{a, k} - U^{t}_{b, k}, M^{k}_{k, b} - U^{t}_{k, a}\right)}
\end{align*}
Hence the condition of Equation~\ref{lemma.luprojalgo.convergeL.eq2} holds.

For case {\bf C2.1}, \emph{i.e.} $M^k_{a, b} =  M^{k-1}_{a, b} = M^{0}_{a, b}$. Then, 
\begin{align*}
M^k_{a, b} &=  M^{k-1}_{a, b} \\
             &= M^{0}_{a, b}
\end{align*}
Hence the condition of Equation~\ref{lemma.luprojalgo.convergeL.eq2} holds.

For case {\bf C2.2}, \emph{i.e.} $M^k_{a, b} =  M^{k-1}_{a, b} = \max {\left(M^{k-1}_{a, e} - U^{t}_{b, e}, M^{k-1}_{e, b} - U^{t}_{e, a}\right)} \text{ where } e \in [k-1]$. Then, if we assume that
\begin{align*}
\max {\left(M^{k-1}_{a, e} - U^{t}_{b, e}, M^{k-1}_{e, b} - U^{t}_{e, a}\right)} = \max {\left(M^{k}_{a, e} - U^{t}_{b, e}, M^{k}_{e, b} - U^{t}_{e, a}\right)},
\end{align*} 
we have:
\begin{align*}
M^k_{a, b}  &=  \max {\left(M^{k-1}_{a, e} - U^{t}_{b, e}, M^{k-1}_{e, b} - U^{t}_{e, a}\right)}\\
			 &= \max {\left(M^{k}_{a, e} - U^{t}_{b, e}, M^{k}_{e, b} - U^{t}_{e, a}\right)}
\end{align*}
Hence the condition of Equation~\ref{lemma.luprojalgo.convergeL.eq2} holds. We conclude by showing that the above assumption must be true for {\bf C2.2}. 
We show this by contradiction as follows. For the sake of contradiction, let us assume that
\begin{align*}
\max {\left(M^{k-1}_{a, e} - U^{t}_{b, e}, M^{k-1}_{e, b} - U^{t}_{e, a}\right)} < \max {\left(M^{k}_{a, e} - U^{t}_{b, e}, M^{k}_{e, b} - U^{t}_{e, a}\right)}
\end{align*} 

Then, 
\begin{align*}
M^k_{a, b}  &=  \max {\left(M^{k-1}_{a, e} - U^{t}_{b, e}, M^{k-1}_{e, b} - U^{t}_{e, a}\right)}\\
			 &< \max {\left(M^{k}_{a, e} - U^{t}_{b, e}, M^{k}_{e, b} - U^{t}_{e, a}\right)} \\
\implies M^k_{a, b} &< \max {\left(M^{k}_{a, e} - U^{t}_{b, e}, M^{k}_{e, b} - U^{t}_{e, a}\right)}
\end{align*}

This is a contradiction to the proof we did above showing that $M^k$ satisfies Equation~\ref{lemma.luprojalgo.convergeL.eq1}. Hence, it holds that $M^k_{a, b}  = \max {\left(M^{k}_{a, e} - U^{t}_{b, e}, M^{k}_{e, b} - U^{t}_{e, a}\right)}$.

Now, putting this all together, we have proved that $M^k$ satisfies the conditions of Equations~\ref{lemma.luprojalgo.convergeL.eq1},~\ref{lemma.luprojalgo.convergeL.eq2}. Hence  by induction, the final output $M^n$ of the algorithm \lproj satisfies these conditions as well, which is exactly equivalent to the claim of the lemma we want to prove.
\end{proof}


%
%
%


\begin{proof}[\bfseries{Proof of Theorem~\ref{thm.luprojalgo}}]
We want to prove that the lower and upper bounds $L^t, U^t$ returned by \luproj are the unique optimal solution of Problem~\ref{eq.jointprojection}. Given the results from Theorem~\ref{thm.projection}, it is enough to show that the upper bound matrix $U^t$ is the optimal solution of Problem~\ref{sec.algoPart2-LUBounds.proj.5} and the lower bound matrix $L^t$ is the optimal solution of Problem~\ref{sec.algoPart2-LUBounds.proj.6}.


We will prove this statement of optimality separately for $U^t$ and for $L^t$. The approach to prove optimality of upper bounds $U^t$ is similar in spirit to that of \citet{brickell2008metric} who showed the optimality of downward-only projection for the metric-nearness problem. For our optimality proofs, we will use the specific properties of $U^t$ and $L^t$ from Lemmas~\ref{lemma.luprojalgo.convergeU} and~Lemma~\ref{lemma.luprojalgo.convergeL}, respectively.

{\bfseries Proof of optimality of upper bounds $U^t$.}\\
We will prove by contradiction. Consider the optimal solution of Problem~\ref{sec.algoPart2-LUBounds.proj.5} given by $U^{*t}$. Let us assume that $U^t$ is not an optimal solution. Let us sort the indices $(i,j) \in [n]^2$ in order of increasing values of $U^t_{i,j}$. For the sake of contradiction, let $(a,b)$ be the first pair of indices in this sorted order such that $U^{*t}_{a,b} > U^t_{a,b}$. We will show that no such pair of indices $(a,b)$ can exist.

By Lemma~\ref{lemma.luprojalgo.convergeU}, we know that $U^t$ returned by \uproj satisfies either:
\begin{enumerate}[label={\bf C{\arabic*}}:]
\item $U^t_{a,b} = \widetilde{U}^t_{a,b}$ or
\item $\exists k \in [\num] \text{ s.t. } U^t_{a,b} = U^t_{a,k} + U^t_{k,b}$
\end{enumerate}

Considering the case {\bf C1}, we have the following contradiction:
\begin{align*}
U^t_{a,b} = \widetilde{U}^t_{a,b} \geq U^{*t}_{a,b}
\end{align*}

Considering the case {\bf C2}, we have the following contradiction:
\begin{align*}
U^t_{a,b} = U^t_{a,k} + U^t_{k,b} \geq U^{*t}_{a,k} + U^{*t}_{k,b} \geq U^{*t}_{a,b},
\end{align*}
where the first inequality holds because of the considered order on the pair of indices.

Hence, we must have $U^t \geq U^{*t}$ showing that $U^t$ is an optimal solution. In fact, given the uniqueness of the optimal solution $U^{*t}$ for Problem~\ref{sec.algoPart2-LUBounds.proj.5} which follows from results of Lemma~\ref{lemma.projection.valid2}, we have  $U^t = U^{*t}$.

{\bfseries Proof of optimality of lower bounds $L^t$.}\\
The ideas are similar as those used above for proving the optimality of $U^t$. Again, we will prove by contradiction. Consider the optimal solution of Problem~\ref{sec.algoPart2-LUBounds.proj.6} given by $L^{*t}$. Let us assume that $L^t$ is not an optimal solution. Let us sort the indices $(i,j) \in [n]^2$ in order of decreasing values of $L^t_{i,j}$. For the sake of contradiction, let $(a,b)$ be the first pair of indices in this sorted order such that $L^{*t}_{a,b} < L^t_{a,b}$. By will show that it is not possible for this pair of indices $(a,b)$ to exist.

By Lemma~\ref{lemma.luprojalgo.convergeL}, we know that $L^t$ returned by \lproj satisfies either:
\begin{enumerate}[label={\bf C{\arabic*}}:]
\item $L^t_{a,b} = \widetilde{L}^t_{a,b}$ or
\item $\exists k \in [\num] \text{ s.t. } L^{t}_{a, b} = \max {\left(L^{t}_{a, k} - U^{t}_{b, k}, L^{t}_{k, b} - U^{t}_{k, a}\right)}$
\end{enumerate}

Considering the case {\bf C1}, we have the following contradiction:
\begin{align*}
L^t_{a,b} = \widetilde{L}^t_{a,b} \leq L^{*t}_{a,b}
\end{align*}

Considering the case {\bf C2}, we have the following contradiction:
\begin{align*}
L^t_{a,b}  &= \max {\left(L^{t}_{a, k} - U^{t}_{b, k}, L^{t}_{k, b} - U^{t}_{k, a}\right)} \\
             &\leq \max {\left(L^{*t}_{a, k} - U^{t}_{b, k}, L^{*t}_{k, b} - U^{t}_{k, a}\right)} \\
             &= \max {\left(L^{*t}_{a, k} - U^{*t}_{b, k}, L^{*t}_{k, b} - U^{*t}_{k, a}\right)} \\
             &\leq L^{*t}_{a, b},
\end{align*}
where the first inequality holds because of the considered order on the pair of indices.

Hence, we must have $L^t \leq L^{*t}$ showing that $L^t$ is an optimal solution. In fact, given the uniqueness of optimal solution $L^{*t}$ for Problem~\ref{sec.algoPart2-LUBounds.proj.6} which follows from results of Lemma~\ref{lemma.projection.valid3}, we have  $L^t = L^{*t}$.
\end{proof}

}

\end{document}